\documentclass[lettersize,journal]{IEEEtran}
\usepackage{amsmath,amsfonts}
\usepackage{algorithmic}
\usepackage{algorithm}
\usepackage{array}
\usepackage[caption=false,font=normalsize,labelfont=sf,textfont=sf]{subfig}
\usepackage{textcomp}
\usepackage{stfloats}
\usepackage{url}
\usepackage{verbatim}
\usepackage{graphicx}
\usepackage{cite}
\usepackage{booktabs}
\usepackage{multirow}
\usepackage{makecell}
\usepackage[table]{xcolor}
\usepackage[most]{tcolorbox}
\usepackage{xcolor}
\usepackage{graphicx}
\usepackage{array}
\usepackage{pifont}

\usepackage{amsthm}
\newtheorem{assumption}{Assumption}
\newtheorem{definition}{Definition}

\newtheorem{theorem}{Theorem}
\usepackage{arydshln}

\definecolor{PromptGrayBg}{RGB}{247,247,247}
\definecolor{PromptGrayFrame}{RGB}{185,185,185}
\definecolor{DeepOrange}{RGB}{213, 87, 0}
\definecolor{DeepPink}{RGB}{205, 90, 140}

\newtcolorbox{promptbox}[3][]{%
    enhanced,
    colback=#2,
    colframe=#3,
    coltitle=black,
    colbacktitle=#2,
    fonttitle=\bfseries,
    fontupper=\small,
    title={#1},
    boxrule=0.6pt,
    arc=1mm,
    left=2mm,
    right=2mm,
    top=1.5mm,
    bottom=1.5mm,
    before skip=0pt,
    after skip=0pt
}

\begin{document}

\title{From Glance to Focus: Mitigating Recognition Laziness in MLLMs for Fine-Grained Recognition}

\author{Yutong Yang, Zhuqi Liu, Mouxing Yang, and Xi Peng%
\thanks{Y. Yang, Z. Liu, and M. Yang are with the College of Computer Science, Sichuan University, Chengdu 610065, China
(e-mail: \{yutongyang.gm, yunfree36, yangmouxing\}@gmail.com).
X. Peng is with the School of Artificial Intelligence, Sichuan University, Chengdu 610065, China,
and also with Tianfu Jincheng Laboratory, Chengdu 610093, China
(e-mail: pengx.gm@gmail.com).
}%
}


\maketitle
            
\begin{abstract}
Multimodal Large Language Models (MLLMs) generally fall short of expectations in Fine-Grained Visual Recognition (FGVR), although they exhibit promising FGVR foundations, \textit{i.e.}, visual perception and broad category knowledge.
Most existing studies seek to remedy this deficiency by adopting domain-specific post-training, at the risk of compromising generalization to unseen domains.
Different from these efforts, in this paper, we revisit the FGVR mechanism underlying MLLMs from both the empirical and theoretical perspectives and uncover an inference-time phenomenon termed \textit{FGVR Laziness}.
In brief, FGVR Laziness refers to the insufficient association between visual cues and category concepts under certain candidate category lists and recognition guidance, leaving the FGVR foundations of MLLMs largely under-utilized.
To mitigate FGVR Laziness, we propose Laziness-resIstant Visual rEcognition (LIVE), a training-free framework which jointly refines candidate lists and recognition guidance to elicit the latent FGVR capability of MLLMs.
Specifically, LIVE first performs query-adaptive candidate inference to derive a compact yet reliable candidate list.
Meanwhile, LIVE constructs category-specific FGVR skills by mining the intra-category characteristics and inter-category differences.
Finally, LIVE selectively invokes detailed skill guidance according to recognition uncertainty and category confusion, shifting MLLM-based FGVR from a cursory glance to a focus on associating discriminative visual cues with category concepts.
Extensive experiments on 13 benchmarks against 24 state-of-the-art methods spanning fine-grained category recognition, action recognition, and general category recognition, verify the effectiveness and generalizability of LIVE across different MLLM backbones. 
The code will be released upon acceptance.
\end{abstract}


\section{Introduction}
As a long-standing challenge in computer vision, Fine-Grained Visual Recognition (FGVR) aims to distinguish subordinate categories characterized by nuanced visual differences, thereby supporting a wide range of applications, including but not limited to species identification~\cite{inaturalist}, activity analysis~\cite{finegym}, intelligent surveillance~\cite{yang2024robust}, and medical diagnosis~\cite{dermatologist}.
The key to FGVR lies in pinpointing subtle visual cues in the query image with respect to each candidate category and associating them with category-specific concepts.
At the intersection of vision foundation models and large language models, Multimodal Large Language Models (MLLMs)~\cite{qwen2.5-vl,llava-onevision,internvl3} naturally combine promising perceptual capability with broad-spectrum general knowledge, thus being expected to provide a promising avenue for FGVR.
However, recent studies~\cite{whybad,finedefics} have demonstrated that most existing MLLMs still fall short of expectations in FGVR and even underperform their vision counterparts, such as CLIP~\cite{clip} and SigLIP~\cite{siglip}. 

To tackle this dilemma, existing studies generally fall into the following two paradigms. 
To be specific, the perception-enhanced paradigm typically employs prompt optimization~\cite{autosep} or reinforcement learning~\cite{dive-k,fine-r1} to strengthen the fine-grained perception capability of MLLMs, enabling them to inspect local regions, compare discriminative features, and thus infer categories from subtle visual cues.
In contrast, the knowledge-enhanced paradigm aims to strengthen the cognition capability of MLLMs regarding subordinate categories by resorting to large-scale fine-grained annotated data~\cite{FOCI} or expert-level domain knowledge, such as detailed category descriptions~\cite{finedefics} and hierarchical category information~\cite{TARA}.
Despite their notable progress, most existing approaches rely heavily on post-training to construct FGVR-oriented MLLMs.
As a result, the acquired FGVR capability tends to be tied to specific training recipes and seen domains, which inevitably limits the generalization of MLLMs to unseen fine-grained scenarios, as verified in our experiments.

Different from existing studies that exhaustively adopt post-training to improve the FGVR capability of MLLMs, in this paper, we first revisit the FGVR mechanism underlying MLLMs.
As shown in Fig.~\ref{fig:motivation}(a), MLLMs can not only effectively perceive subtle visual cues but also demonstrate a sophisticated understanding of category-specific concepts.
Nevertheless, these foundational capabilities might not be fully elicited under certain recognition conditions, a phenomenon we refer to as \textit{FGVR Laziness}.
To examine the factors underlying FGVR Laziness, we present the following three observations in Fig.~\ref{fig:motivation}(b)-(c).
First, candidate-wise visual inspection is sensitive to presentation position, as moving the ground-truth category forward generally increases category-to-image attention and improves recognition accuracy.
Second, when the candidate list is overly large or contains highly confusable categories, MLLMs might settle on plausible yet incorrect predictions, leaving the ground-truth category insufficiently associated with the query image.
Intuitively, if these prediction errors arose solely from an inability to distinguish the correct category from the originally predicted category, modifying the remaining candidates alone would be unlikely to resolve the confusion.
Surprisingly, simply replacing or removing the other candidates could correct a substantial proportion of incorrect predictions, even when the originally predicted category remains among the candidates.
Furthermore, the accompanying increase in category-to-image attention suggests that the available visual cues and category concepts can be more effectively associated under appropriate recognition conditions.
Third, introducing explicit inter-category comparison guidance could also correct erroneous predictions, steering MLLMs to inspect the visual differences needed to distinguish highly confusable categories.

Beyond the above empirical studies, we further conduct comprehensive theoretical analyses of FGVR Laziness.
These empirical and theoretical findings together suggest that FGVR Laziness hinders the utilization of FGVR capabilities rooted in MLLMs, highlighting the joint refinement of candidate lists and recognition guidance as a promising avenue for mitigating FGVR Laziness and eliciting the latent FGVR capabilities.
However, translating these insights into practice remains non-trivial, raising three key challenges.
i) Which candidates should be retained? 
Since the ground-truth category is unavailable at inference time, candidate list construction must balance compactness against reliable ground-truth coverage. 
ii) How to derive the comparison guidance? 
General descriptions are often insufficient, calling for effective guidance that captures visually grounded differences between confusable categories. 
iii) When and for which categories should detailed guidance be invoked?
To avoid redundant context and additional interference, detailed guidance should be invoked only for uncertain queries and their most confusable candidates.

\begin{figure*}[t]
    \centering
    \includegraphics[width=\linewidth]{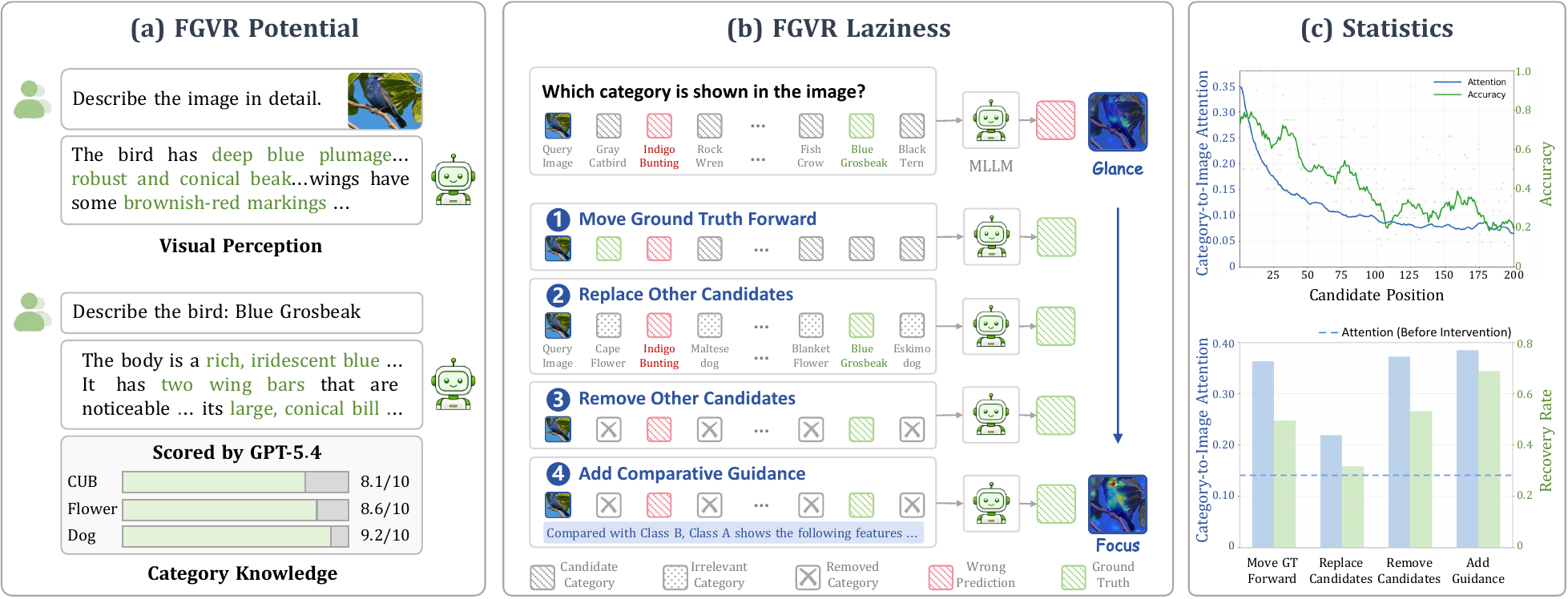}
    \caption{
    \textbf{Our motivations and observations.}
    \textbf{(a) FGVR Potential}.
    MLLMs can describe fine-grained visual cues and category-specific characteristics, demonstrating promising perceptual and conceptual capabilities required for FGVR.
    \textbf{(b) FGVR Laziness}. 
    FGVR Laziness refers to the inference-time phenomenon where MLLMs might fail to consistently inspect the query image with respect to all plausible candidate categories.
    To investigate this phenomenon, we conduct the following two kinds of controlled experiments.
    i) Position variation: we move the ground-truth category forward in the candidate list while keeping the other candidates unchanged (\ding{172}).
    ii) Candidate-competition variation: we keep the ground-truth category and the original mispredicted category unchanged, while either replacing the remaining candidates with irrelevant categories (\ding{173}) or removing them altogether (\ding{174}).
    Furthermore, we introduce explicit inter-category comparison guidance to distinguish the ground-truth category from the original mispredicted category (\ding{175}). 
    \textbf{(c) Statistics}.
    Quantitative results show that moving the ground-truth category forward, reducing candidate competition, and introducing comparison guidance generally increase category-to-image attention and improve recognition.
    In this figure, \textit{Recovery Rate} denotes the proportion of initially incorrect predictions that are corrected under each setting.
    The attention score measures category-to-image attention averaged over ground-truth category tokens.
    Qwen2.5-VL-7B serves as the default MLLM, while GPT-5.4 is employed to evaluate the category-specific descriptions generated by the MLLM.
}
    \label{fig:motivation}
\end{figure*}

Based on the above observations and discussions, we propose Laziness-resIstant Visual rEcognition (LIVE), a novel training-free framework tailored to alleviate FGVR Laziness and elicit the latent FGVR capability of MLLMs.
In brief, LIVE consists of three core modules, namely, Plausible Candidate Inference (PCI), Category Skill Construction (CSC) and Skill-inspired Confusion Resolution (SCR), which respectively address the aforementioned three key challenges. 
To begin with, the PCI module jointly models query-specific visual relevance and inter-category affinity, thus narrowing the full category list down to a compact candidate list while maintaining reliable ground-truth coverage.
Meanwhile, the CSC module performs self-reflection to summarize intra-category visual cues and inter-category differences, thereby forming category-specific and reusable FGVR skills.
Based on the constructed candidate list and FGVR skills, the SCR module invokes detailed skill guidance on demand based on recognition uncertainty and category confusion, steering MLLMs toward inspecting discriminative regions and cues that distinguish the most confusable candidates.
Overall, by dexterously integrating these three modules, LIVE shifts recognition from a cursory glance over plausible candidates to a focus on associating discriminative visual cues with category concepts, thus achieving on-the-fly FGVR Laziness mitigation at inference time.

The main contributions of this paper are summarized as follows:
\begin{itemize}
    \item We reveal that MLLMs inherently possess promising FGVR potential while being largely impeded by FGVR Laziness.
    As a remedy, our empirical and theoretical analyses establish the joint refinement of candidate lists and recognition guidance as a promising avenue for mitigating FGVR Laziness, broadening the perspective of MLLM-based FGVR studies from post-training-based capability acquisition to latent capability elicitation.
    \item We propose LIVE, an on-the-fly FGVR Laziness mitigation approach that elicits the latent FGVR capabilities of off-the-shelf MLLMs. 
    In brief, by integrating the PCI, CSC, and SCR modules, LIVE shifts recognition from cursory inspection of plausible candidates to focused verification of the discriminative cues that distinguish them. As a result, LIVE enables effective FGVR without additional training.
    \item We conduct extensive experiments on three representative recognition tasks, namely, fine-grained image recognition, fine-grained action recognition, and general image recognition. Results on $13$ benchmarks against $24$ competitive methods verify the effectiveness of LIVE. Furthermore, we demonstrate that LIVE can serve as a plug-and-play solution for improving FGVR across different off-the-shelf MLLMs.
\end{itemize}
\section{Related Work}

In this section, we briefly review three topics highly related to this work, including vanilla fine-grained visual recognition, MLLM-based fine-grained visual recognition, and attention bias in MLLMs.

\subsection{Vanilla Fine-Grained Visual Recognition}

Fine-Grained Visual Recognition (FGVR) aims to distinguish subordinate categories with highly similar overall appearances but subtle visual differences. 
To this end, conventional FGVR methods~\cite{part-rcnn,ps-cnn} generally rely on fine-grained annotations to learn task-specific discriminative representations.
Despite their effectiveness, these methods typically rely on task-specific supervision, making the resulting recognition capability closely tied to predefined categories and target domains. 
As a remedy, Vision-Language Models (VLMs)~\cite{clip} have emerged as a more general recognition paradigm by aligning images and category semantics in a shared embedding space. 
Specifically, each category can be represented by a simple textual prompt and recognized through image-text similarity matching.
To further adapt pretrained VLMs to fine-grained scenarios, existing approaches could be broadly grouped into the following two categories: 
i) task-adaptation methods, which leverage task-specific samples to refine the prediction by optimizing prompts~\cite{coop,maple,CoCoOp}, introducing lightweight adapters~\cite{clip-adapter,mma}, or constructing non-parametric caches~\cite{tip-adapter}; 
ii) LLM-enhanced methods, which exploit the rich knowledge of LLMs to generate category-specific descriptions~\cite{dclip,cupl} or automatically optimize discriminative textual prompts~\cite{proapo}, thereby enriching textual representations for fine-grained image-text matching.

Among the above approaches, LLM-enhanced methods are most relevant to our work but follow a different recognition paradigm.
Specifically, these methods mainly employ LLMs to generate category-specific knowledge for enriching textual representations, while relying on VLM-based image-text matching for final recognition.
In contrast, MLLMs integrate visual perception and category knowledge within a unified model, enabling direct reasoning over visual inputs and the corresponding category concepts.
Building upon this paradigm, we revisit the limited FGVR performance of MLLMs and reveal an under-explored inference-time phenomenon termed FGVR Laziness, where off-the-shelf MLLMs possess promising FGVR foundations yet fail to sufficiently utilize them during recognition.
To mitigate FGVR Laziness, we propose LIVE, an on-the-fly approach that elicits the latent FGVR capability of MLLMs, thereby enabling effective fine-grained recognition without additional training.

\subsection{MLLM-based Fine-Grained Visual Recognition}

Large-scale multimodal pre-training has endowed MLLMs with strong visual perception and broad category knowledge, making them promising for FGVR.
Nevertheless, directly applying general-purpose MLLMs to FGVR still yields unsatisfactory performance.
To improve MLLM-based FGVR, existing approaches could be broadly grouped into the following two paradigms according to the capability they aim to strengthen: 
i) perception-enhanced methods, which improve fine-grained visual perception by promoting localized and comparative visual reasoning through prompt optimization~\cite{autosep} or task-specific post-training, such as supervised fine-tuning and reinforcement learning~\cite{fine-r1,dive-k}.
ii) knowledge-enhanced methods, which strengthen the cognition of subordinate categories by enriching MLLMs with fine-grained category supervision~\cite{FOCI}, detailed attribute knowledge~\cite{finedefics}, or hierarchical category structures~\cite{TARA}.

The major differences between existing MLLM-based FGVR approaches and our work are two-fold. 
On the one hand, existing methods mainly attribute limited FGVR performance to capability deficiency and accordingly strengthen either fine-grained perception or category cognition. 
In contrast, our work reveals an under-explored inference-time phenomenon termed FGVR Laziness, where MLLMs already possess the perceptual and conceptual capabilities required for FGVR yet fail to sufficiently utilize them.
On the other hand, most existing approaches acquire FGVR-oriented capabilities through task-specific post-training, whereas our work mitigates FGVR Laziness at inference time without additional task-specific training. 
Overall, our work suggests that latent capability elicitation at inference time could serve as a complementary approach to post-training-based capability acquisition for improving MLLM-based FGVR.

\subsection{Attention Bias in MLLMs}

Attention serves as an important mechanism for modeling token interactions and aggregating contextual information during inference, thereby influencing the relative contributions of different input components to model predictions.
Recent studies have shown that attention allocation varies with both input characteristics and the inference process, giving rise to various forms of attention bias.
Specifically, existing studies can be broadly grouped into the following two categories:
i) input-dependent attention bias, which arises from characteristics of the input, including position bias~\cite{multi-image-position}, excessive attention to specific visual tokens~\cite{visual-attention-sink}, and imbalanced attention between visual and textual modalities~\cite{seeing-clearly};
ii) inference-dependent attention bias, which emerges as inference proceeds, including progressively weakened visual attention during generation~\cite{sparc,seeing-far} and increasing dispersion of attention away from task-relevant visual regions during multi-step reasoning~\cite{deeper-thought}.

Among the above studies, input-dependent attention bias is most relevant to our work, as FGVR Laziness is sensitive to the recognition inputs, including the candidate list and recognition guidance.
Existing studies mainly focus on positional bias, which could partially explain the position-related observation in Fig.~\ref{fig:motivation}.
However, even when the target category and its position remain unchanged, merely replacing or removing the other candidates could substantially alter its category-to-image attention and even reverse the original prediction, revealing a candidate competition effect beyond positional bias in MLLM-based FGVR.
Furthermore, different from prior studies that mainly characterize attention biases, our work investigates their connection to FGVR Laziness.
Specifically, our empirical and theoretical analyses show that candidate lists and recognition guidance can jointly affect attention allocation and recognition results, which motivates us to refine these two factors for FGVR Laziness mitigation.

\section{Method}
\label{sec:method}

\begin{figure*}[t]
    \centering
    \includegraphics[width=\linewidth]{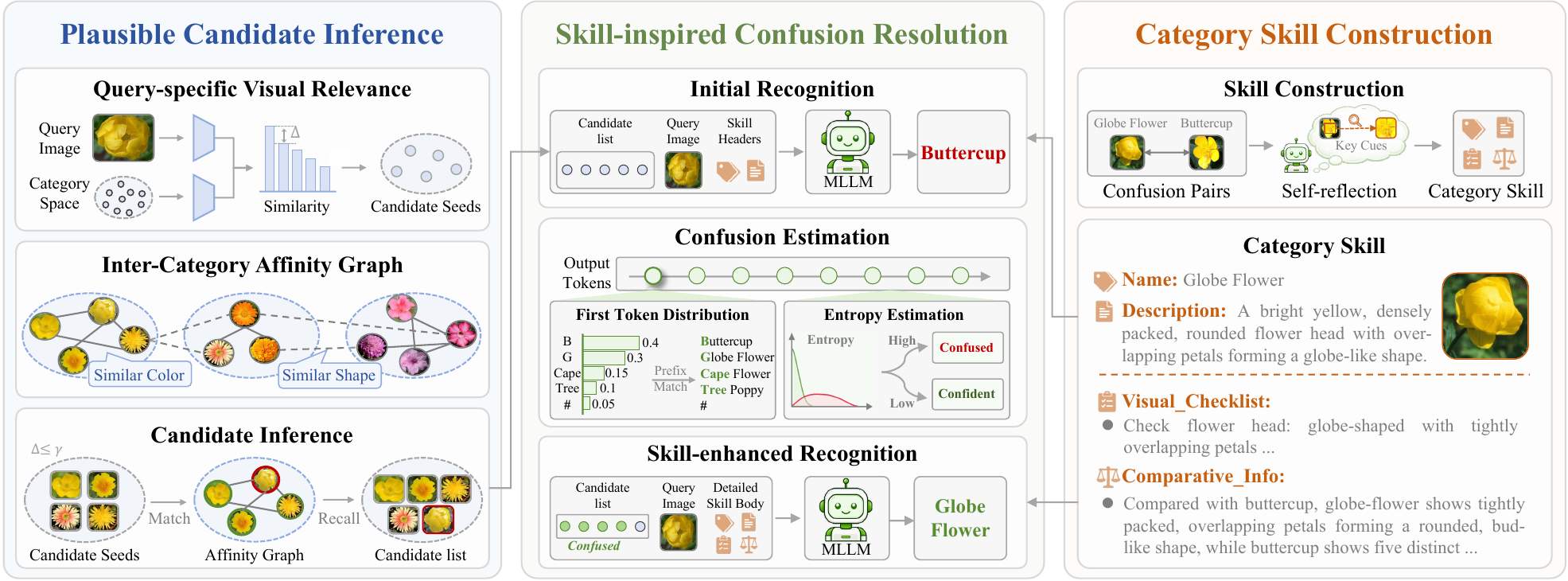}
\caption{
\textbf{Overview of the proposed LIVE.}
Based on the empirical observations and theoretical analyses, LIVE aims to mitigate FGVR Laziness by reducing candidate competition and providing adaptive recognition guidance.
Specifically, for a given query image, Plausible Candidate Inference (PCI) first combines query-specific visual relevance with a pre-constructed inter-category affinity graph to infer a compact candidate list with reliable ground-truth coverage.
Meanwhile, Category Skill Construction (CSC) leverages MLLM self-reflection to construct reusable Category Skills, consisting of lightweight headers and detailed visual and comparative guidance.
Based on the inferred candidates and Category Skills, Skill-inspired Confusion Resolution (SCR) performs initial recognition with skill headers and estimates prediction uncertainty from the output-token distribution.
For uncertain predictions, SCR further identifies confused categories through first-token prefix matching and adaptively invokes the corresponding detailed guidance for skill-enhanced recognition.
In the figure, for illustration, we present the LIVE pipeline in the image-based FGVR setting using examples from the Flower dataset.
}
    \label{fig:overview}
\end{figure*}

In this section, we elaborate on the proposed training-free framework dubbed Laziness-resIstant Visual rEcognition (LIVE) for mitigating FGVR Laziness at inference time.
As shown in Fig.~\ref{fig:overview}, LIVE consists of Plausible Candidate Inference (PCI) to construct a compact yet reliable candidate list, Category Skill Construction (CSC) to construct reusable recognition guidance, and Skill-inspired Confusion Resolution (SCR) to invoke relevant guidance on demand.
In the following, we first formulate FGVR Laziness and analyze potential mitigation strategies in Section~\ref{sec:problem_formulation}, then introduce the PCI module in Section~\ref{sec:candidate_inference}, and finally elaborate on the CSC and SCR modules in Sections~\ref{sec:category_skill} and~\ref{sec:skill_elicitation}, respectively.

\subsection{Problem Formulation}
\label{sec:problem_formulation}

Given a query image/video $x$, MLLM-based FGVR aims to identify its fine-grained category $\hat{y}$ from a candidate list $\mathbf{C}=[c_1,\ldots,c_M]$ via the MLLM $\mathcal{M}_{\Theta}$.
Mathematically, the above FGVR process could be formulated as
\begin{equation}
\hat{y}
=
\mathcal{M}_{\Theta}(x, \mathbf{C}, G),
\qquad
\hat{y}\in\mathbf{C},
\label{eq:fgvr_formulation}
\end{equation}
where $G$ denotes the task instructions that guide the  MLLM to distinguish among candidate categories by associating fine-grained visual cues in $x$ with the corresponding category concepts in $\mathbf{C}$.

As discussed in the Introduction, numerous FGVR errors produced by MLLMs could be corrected by adjusting candidate ordering, candidate composition, or recognition guidance.
These observations suggest that, despite being capable of perceiving fine-grained visual cues and understanding category concepts, MLLMs might fail to effectively associate the two under certain recognition conditions, thus leading to incorrect FGVR predictions.
Accordingly, we refer to the above insufficient association between available visual cues and category concepts as FGVR Laziness, and formally define it as follows.

\begin{definition}[FGVR Laziness]
\label{def:fgvr_laziness}
For a query sample $x$ with ground-truth category $y$ and an incorrect prediction $\hat{y}=\mathcal{M}_{\Theta}(x, \mathbf{C}, G)$ ($\hat{y}\neq y$),
an MLLM exhibits FGVR Laziness if there exist admissible variations $(\mathbf{C}',G')$ such that
\begin{equation}
y,\hat{y}\in\mathbf{C}',
\qquad
\mathcal{M}_{\Theta}(x, \mathbf{C}', G')=y.
\label{eq:fgvr_laziness}
\end{equation}
Here, admissible variations should preserve the recognition task without revealing the correct answer, while keeping the model and the query sample unchanged.
\end{definition}

Based on the above discussion and definition, we aim to mitigate FGVR Laziness by deriving desirable candidate lists $\mathbf{C}$ or recognition guidance $G$, thus facilitating the association between visual cues and category concepts.
To this end, we first investigate what properties of candidate lists and recognition guidance could strengthen such associations.
Since the association strength cannot be directly observed, we introduce the following assumption.

\begin{assumption}
\label{ass:attention_proxy}
The association strength between visual cues and the concept of a given category $c$ within an MLLM is positively correlated with category-to-image attention $R_{\Theta}(c\mid x,\mathbf{C},G)$, with $\Theta$ and $x$ fixed.
\end{assumption}

The detailed formulation of $R_{\Theta}(c\mid x,\mathbf{C},G)$ is provided in the Supplementary Material.
The rationale for the above assumption is twofold.
On the one hand, the attention mechanism captures interactions between visual and category representations, which could naturally reflect the associations between visual cues and category concepts.
On the other hand, our observations in Fig.~\ref{fig:motivation}(c) demonstrate that increased category-to-image attention consistently accompanies improved recognition of the ground-truth category.
Accordingly, we adopt category-to-image attention $R_{\Theta}(c\mid x,\mathbf{C},G)$ as a proxy for the association strength in the subsequent analysis.
To begin with, we investigate the role of the candidate list $\mathbf{C}$ by examining how candidate ordering and candidate competition affect category-to-image attention.
Specifically, we consider the causal attention setting where image tokens precede the candidate list and derive the following two theorems.

\begin{theorem}[Candidate Ordering Effect]
\label{thm:ordering}
Under causal attention with image tokens preceding the candidate list and attention logits to shared accessible tokens held fixed, moving a specified category $c$ forward across other candidates increases its category-to-image attention, namely,
\begin{equation}
R_{\Theta}(c\mid x,\mathbf{C}_{c\cdot},G)
>
R_{\Theta}(c\mid x,\mathbf{C}_{\cdot c},G),
\label{eq:ordering_effect}
\end{equation}
where $\mathbf{C}_{c\cdot}$ and $\mathbf{C}_{\cdot c}$ have the same candidate composition but place $c$ at earlier and later positions, respectively.
The relative order of other candidates is preserved, with $x$ and $G$ fixed.
\end{theorem}

\begin{theorem}[Candidate Competition Effect]
\label{thm:composition}
Under causal attention with image tokens preceding the candidate list and attention logits to shared accessible tokens held fixed, reducing attention competition from other candidates increases the category-to-image attention of a specified category $c$, namely,
\begin{equation}
R_{\Theta}(c\mid x,\mathbf{C}_{c+},G)
>
R_{\Theta}(c\mid x,\mathbf{C}_{c-},G),
\label{eq:composition_effect}
\end{equation}
where $\mathbf{C}_{c+}$ and $\mathbf{C}_{c-}$ impose weaker and stronger attention competition on $c$, respectively.
The position of $c$ is preserved, with $x$ and $G$ fixed.
\end{theorem}



In Theorem~\ref{thm:composition}, $\mathbf{C}_{c+}$ corresponds to replacing competing candidates with less related categories.
For example, in fine-grained bird recognition on CUB~\cite{cub}, competing bird categories can be replaced with aircraft categories from FGVC-Aircraft~\cite{aircraft}.
Clearly, directly removing these competing candidates would reduce the candidate space and maximize the attention gain.
Due to space limitations, the proofs of both theorems and the above claim are provided in the Supplementary Material.

Notably, although both candidate reordering and candidate reduction have an incentive to increase category-to-image attention, the two strategies incur different risks.
Specifically, since the ground-truth category $y$ is unknown at inference time, candidate reordering might shift $y$ backward, whereas candidate reduction might exclude $y$.
To quantify these potential adverse effects, we define risk as the probability that an intervention either excludes $y$ or decreases its category-to-image attention, leading to the following theorem.

\begin{theorem}[Degradation Risk of Candidate Interventions]
\label{thm:risk}
Consider a uniform distribution over the $M$ candidate positions.
Move-to-front reordering has an average degradation risk of $\frac{M-1}{2M}$ over all possible starting positions, whereas order-preserving candidate reduction has a risk of $1-p_{\mathrm{ret}}$.
Therefore, candidate reduction has lower degradation risk if and only if
\begin{equation}
p_{\mathrm{ret}}>\frac{1}{2}+\frac{1}{2M},
\label{eq:risk_comparison}
\end{equation}
where $p_{\mathrm{ret}}$ denotes the ground-truth retention probability.
\end{theorem}


Since FGVR typically involves a large number of candidate categories, the retention threshold in Theorem~\ref{thm:risk} approaches $50\%$ as $M$ increases.
Moreover, for a fixed $M$, the average degradation risk of move-to-front reordering is fixed under the above setting, whereas the risk of candidate reduction decreases as the ground-truth retention probability increases.
Accordingly, we focus on candidate reduction and develop Plausible Candidate Inference (PCI) to construct a compact candidate list while maintaining reliable ground-truth coverage.

As for recognition guidance $G$, the observations in Fig.~\ref{fig:motivation} show that adjusting $G$ can correct predictions while keeping the candidate list unchanged.
However, the guidance needed for effective recognition varies across queries with different levels of recognition difficulty and category confusion.
Specifically, some queries can be resolved with basic category information, whereas others require more detailed visual and comparative cues.
This calls for guidance that supports different levels of detail and can be invoked adaptively according to query-specific recognition difficulty and category confusion.
To this end, we further develop Category Skill Construction (CSC) and Skill-inspired Confusion Resolution (SCR) to provide adaptive recognition guidance.

\subsection{Plausible Candidate Inference}
\label{sec:candidate_inference}

The above analysis shows that candidate reduction could mitigate FGVR Laziness as long as the ground-truth category can be reliably retained.
This motivates us to construct a compact yet reliable candidate list.
A straightforward strategy is to retain the Top-$k$ categories according to their visual relevance to the query image.
However, fine-grained categories often exhibit subtle visual differences, and the ground-truth category might rank below visually similar categories.
Consequently, a small $k$ might exclude the ground truth, whereas a large $k$ retains excessive competing categories.
To address this dilemma, PCI combines query-specific visual relevance with inter-category affinity to perform query-adaptive candidate inference.

Specifically, PCI first employs an off-the-shelf vision-language model (VLM) to estimate the visual relevance of each category $c_i$ to a query sample $x$ as
\begin{equation}
\label{eq:visual_relevance}
s_i(x)
=
\operatorname{sim}
\left(
f_v(x),f_t(c_i)
\right),
\end{equation}
where $f_v(\cdot)$ and $f_t(\cdot)$ denote the visual and textual encoders of the VLM, respectively, and $\operatorname{sim}(\cdot,\cdot)$ denotes the model-specific image-text similarity score, such as cosine similarity for CLIP.
We rank all categories in descending order of $s_i(x)$, with $c_{(k)}(x)$ and $s_{(k)}(x)$ denoting the $k$-th ranked category and its relevance score, respectively.

However, query-specific visual relevance alone might be unreliable when fine-grained categories are visually similar.
To compensate for this limitation, PCI further constructs an inter-category affinity graph to capture the confusion relations among these visually similar categories.
In the affinity graph, each category is represented as a node, and the directed edge from $c_i$ to $c_j$ reflects how strongly $c_i$ tends to be confused with $c_j$.
Specifically, for each category $c_i$, we collect a set of reference samples $\mathcal{R}_i=\{x_i^1,\ldots,x_i^{N}\}$, where $N$ is fixed to $4$ in our experiments.
For each reference image $x_i^r\in\mathcal{R}_i$, PCI ranks all categories according to Eq.~\ref{eq:visual_relevance} and denotes the Top-$B$ categories as $\mathbf{C}_i^r$.
The affinity from $c_i$ to $c_j$ is then defined as
\begin{equation}
\label{eq:inter_category_affinity}
\operatorname{Aff}(c_i,c_j)
=
\frac{1}{N}
\sum_{r=1}^{N}
\mathbb{I}
\left[
c_j \in \mathbf{C}^{r}_{i}
\right],
\end{equation}
where $\mathbb{I}[\cdot]$ denotes the indicator function, taking the value $1$ if the condition is satisfied and $0$ otherwise.

Thanks to the relevance ranking and affinity graph, PCI could select a small number of highly-ranked categories as seeds and supplement them with confusing categories identified from the affinity graph.
The former controls the candidate size, while the latter improves ground-truth coverage, jointly enabling compact yet reliable candidate inference.
Instead of using a fixed number of seeds for all queries, it is more desirable to adaptively adjust it according to the reliability of the relevance ranking.
To quantify this reliability, PCI uses the gap between the two top-ranked categories:
\begin{equation}
\label{eq:visual_margin}
\delta(x)
=
s_{(1)}(x)-s_{(2)}(x).
\end{equation}
As supported by our empirical analysis in Fig.~\ref{fig:pci_gap} of the Supplementary Material, larger relevance gaps generally indicate more reliable rankings, whereas smaller gaps indicate less reliable ones.
Accordingly, PCI adopts different candidate inference strategies based on $\delta(x)$.

When $\delta(x)>\gamma$, the relevance ranking is considered relatively reliable, and PCI selects the Top-1 category $c_{(1)}(x)$ as the only seed.
After that, PCI supplements the seed with categories that are strongly confused with $c_{(1)}(x)$ in the affinity graph.
Specifically, a category $c_i$ is recalled if $c_{(1)}(x)$ appears in the Top-$B$ lists of all reference samples in $\mathcal{R}_i$, \textit{i.e.}, $\operatorname{Aff}(c_i,c_{(1)}(x))=1$.
Clearly, this condition indicates that $c_i$ is highly confused with $c_{(1)}(x)$, making it a plausible ground-truth category despite the high reliability of the current relevance ranking.

When $\delta(x)\leq\gamma$, the relevance ranking is less reliable, making the Top-1 category alone insufficient as the seed.
For simplicity, PCI adopts the same Top-$B$ setting used for affinity graph construction and takes the Top-$B$ ranked categories as seeds to retrieve plausible categories from the affinity graph.
Specifically, a category $c_i$ is regarded as potentially confusing if it receives affinity support from at least a proportion $\rho$ of the seeds:
\begin{equation}
\label{eq:multi_seed_support}
\frac{1}{B}
\sum_{j=1}^{B}
\mathbb{I}
\left[
\operatorname{Aff}(c_i,c_{(j)}(x))>0
\right]
\geq
\rho.
\end{equation}
For each confusing category $c_i$, PCI further aggregates its affinity to all seeds as
\begin{equation}
\label{eq:affinity_score}
q_i(x)
=
\sum_{j=1}^{B}
\operatorname{Aff}(c_i,c_{(j)}(x)).
\end{equation}
After that, PCI ranks the confusing categories in descending order of $q_i(x)$ and recalls the highest-scoring ones subject to the total candidate budget $K$.

Finally, the resulting candidate list for query $x$ is obtained as
\begin{equation}
\label{eq:low_conf_candidate_list}
\mathbf{C}_x
=
\begin{cases}
[c_{(1)}(x)] \cup \mathbf{C}_x^{\mathrm{H}},
& \delta(x)>\gamma,\\
[c_{(1)}(x),\ldots,c_{(B)}(x)] \cup \mathbf{C}_x^{\mathrm{L}},
& \delta(x)\leq\gamma,
\end{cases}
\end{equation}
where $\mathbf{C}_x^{\mathrm{H}}$ and $\mathbf{C}_x^{\mathrm{L}}$ denote the additional confusing categories recalled under the two conditions, respectively.
Through the above adaptive candidate inference, PCI could significantly reduce category competition while guaranteeing reliable ground-truth coverage.
As shown in Fig.~\ref{fig:cci_analysis}, PCI consistently achieves a ground-truth retention rate above $90\%$ across the evaluated settings, supporting its effectiveness as a low-risk strategy for mitigating FGVR Laziness.

\subsection{Category Skill Construction}
\label{sec:category_skill}

As discussed in Section~\ref{sec:problem_formulation}, it is non-trivial to construct the desirable FGVR guidance.
On the one hand, category confusion varies across queries, making fixed guidance insufficiently targeted. 
On the other hand, constructing tailored guidance from scratch for every query would require repeated inference-time generation.
Recall that recent studies on LLM agents have introduced Agent Skills~\cite{zhang2025agentskills} to encapsulate task-solving knowledge in a hierarchical form, thereby enabling on-demand reuse across tasks~\cite{skilldreamer}. Inspired by this paradigm, CSC formulates fine-grained category knowledge as Category Skills, enabling targeted recognition guidance to be flexibly composed and selectively invoked for different queries.
To the best of our knowledge, this work could be one of the first attempts to introduce the skill paradigm into MLLM-based FGVR.

Following the hierarchical organization of Agent Skills, each Category Skill consists of a lightweight skill header and a detailed skill body.
As shown in Fig.~\ref{fig:overview}, the skill header contains the category name and a concise description, while the skill body contains sophisticated FGVR experience that can be invoked when finer discrimination is needed.
Specifically, the skill body consists of two complementary types of knowledge, namely, category-specific visual characteristics and inter-category comparative knowledge.

To begin with, CSC constructs its category-specific visual knowledge by eliciting the category knowledge rooted in the MLLM and grounding it with a randomly-selected reference image.
Given category $c_i$ and a reference image $x_i^r\in\mathcal{R}_i$, the MLLM generates a concise category description $d_i$ and a visual checklist $\mathbf{v}_i$ using the instruction $q_{\mathrm{cat}}$, which is detailed in the Supplementary Material.
Mathematically,
\begin{equation}
\label{eq:category_visual_skill}
(d_i,\mathbf{v}_i)
=
\mathcal{M}_{\Theta}
\left(
q_{\mathrm{cat}}(c_i,x_i^r)
\right).
\end{equation}
The description $d_i$ summarizes the characteristic appearance of $c_i$ and forms the skill header together with the category name.
The visual checklist $\mathbf{v}_i$ specifies the visual regions and fine-grained attributes to examine, helping the model associate relevant visual cues with corresponding category concepts during recognition.

However, category-specific visual characteristics alone might be insufficient when visually similar categories share many common attributes.
As a remedy, CSC further constructs inter-category comparative knowledge from erroneous MLLM predictions.
Specifically, for each category $c_i$, CSC evaluates its reference samples under the full candidate list $\mathbf{C}$ and collects all incorrectly predicted categories into a confusion set:
\begin{equation}
\label{eq:confusion_set}
\mathcal{Q}_i
=
\left\{
\hat{y}_{\mathbf{C}}(x_i^r)
\mid
x_i^r\in\mathcal{R}_i,\;
\hat{y}_{\mathbf{C}}(x_i^r)\neq c_i
\right\},
\end{equation}
where $\hat{y}_{\mathbf{C}}(x_i^r)$ denotes the MLLM prediction for $x_i^r$ under the full candidate list $\mathbf{C}$.

For each category $c_j\in\mathcal{Q}_i$, CSC uses the instruction $q_{\mathrm{cmp}}$ to guide the MLLM to self-reflect on the confusion between $c_i$ and $c_j$ and identify their discriminative visual cues:
\begin{equation}
\label{eq:comparative_skill}
r_{ij}
=
\mathcal{M}_{\Theta}
\left(
q_{\mathrm{cmp}}(c_i,c_j,x_i^r,x_j^{r'})
\right),
\qquad
c_j\in\mathcal{Q}_i,
\end{equation}
where $x_i^r\in\mathcal{R}_i$ and $x_j^{r'}\in\mathcal{R}_j$ are randomly selected reference samples, and $r_{ij}$ summarizes the discriminative visual differences between $c_i$ and $c_j$.
The comparative knowledge associated with $c_i$ is then collected as $\mathbf{r}_i=\{r_{ij}\}_{c_j\in\mathcal{Q}_i}$.

Finally, the resulting Category Skill for $c_i$ is organized as
\begin{equation}
\label{eq:category_skill}
\mathcal{S}_i
=
\left(
\underbrace{(c_i,d_i)}_{\text{skill header}},
\underbrace{
\left(
\mathbf{v}_i,
\mathbf{r}_i
\right)
}_{\text{skill body}}
\right).
\end{equation}

\subsection{Skill-inspired Confusion Resolution}
\label{sec:skill_elicitation}

Given the candidate list inferred by PCI and the Category Skills constructed by CSC, the remaining question is how to leverage these skills to provide targeted recognition guidance.
A straightforward strategy is to invoke all candidate Category Skills during recognition.
However, our empirical results in Table~\ref{tab:scr_analysis} indicate that such a strategy is not optimal and might even introduce excessive cues that interfere with recognition.
In fact, the key to effective skill utilization is determining when detailed guidance is needed and which Category Skills should be invoked.
To this end, SCR uses prediction uncertainty to determine whether further guidance is necessary and category confusion to determine which skills to invoke.

Specifically, SCR first performs initial recognition using the general recognition instruction together with only the skill headers of the candidate categories.
Accordingly, the initial recognition guidance is constructed as
\begin{equation}
\label{eq:stage1_guidance}
G_x^{(1)}
=
G_g \cup
\left\{
(c_i,d_i)
\mid
c_i\in\mathbf{C}_x
\right\},
\end{equation}
where $G_g$ denotes the general recognition instruction.
Based on $G_x^{(1)}$, the initial prediction is obtained as
\begin{equation}
\label{eq:stage1_recognition}
\hat{y}_x^{(1)}
=
\mathcal{M}_{\Theta}
\left(
x,\mathbf{C}_x,G_x^{(1)}
\right).
\end{equation}

Based on the initial recognition, SCR determines whether detailed skill bodies are needed by estimating prediction uncertainty from the output-token distribution.
Let $p_{\ell}(v)$ and $L_x$ denote the probability of token $v$ at the $\ell$-th decoding position and the number of tokens in the generated category name respectively. SCR computes prediction uncertainty as follows:
\begin{equation}
\label{eq:answer_entropy}
H_x
=
-\frac{1}{L_x}
\sum_{\ell=1}^{L_x}
\sum_{v\in\mathcal{V}}
p_{\ell}(v)
\log p_{\ell}(v),
\end{equation}
where $\mathcal{V}$ denotes the vocabulary of the MLLM.
As shown in Fig.~\ref{fig:answer_entropy}, correct predictions generally exhibit lower entropy than incorrect ones, indicating that $H_x$ can effectively characterize prediction uncertainty.
Based on this observation, SCR introduces an uncertainty threshold $\eta$.
When $H_x\leq\eta$, the initial prediction $\hat{y}_x^{(1)}$ is directly regarded as the final result.
Otherwise, SCR further identifies the confused categories and invokes their corresponding skill bodies on demand.

Specifically, for an uncertain prediction ($H_x > \eta$), SCR identifies confused categories from the first-token distribution.
Since the output is constrained to a candidate category name, high-probability first tokens reflect competing category predictions and can be mapped back to candidate categories through prefix matching.
Therefore, SCR ranks the tokens according to $p_1(v)$ in descending order and retains candidate-matched tokens until the first unmatched token is encountered.
Let $\mathcal{T}_x^{+}$ denote the retained tokens. The corresponding confused categories are obtained as
\begin{equation}
\label{eq:focus_candidates}
\mathbf{C}_x^{\mathrm{conf}}
=
\bigcup_{t\in\mathcal{T}_x^{+}}
\left\{
c_i\in\mathbf{C}_x
\mid
\operatorname{PrefixMatch}(t,c_i)=1
\right\},
\end{equation}
where $\operatorname{PrefixMatch}(t,c_i)$ indicates whether token $t$ matches the prefix of category name $c_i$.
For example, as illustrated in Fig.~\ref{fig:overview}, the tokens ``cape'' and ``tree'' are matched to ``cape flower'' and ``tree poppy'', respectively.
Then, SCR invokes the detailed skill bodies of the confused categories and constructs the skill-enhanced recognition guidance as
\begin{equation}
\label{eq:stage2_guidance}
G_x^{(2)}
=
G_x^{(1)}
\cup
\left\{
\left(
\mathbf{v}_i,
\mathbf{r}_i
\right)
\mid
c_i\in\mathbf{C}_x^{\mathrm{conf}}
\right\}.
\end{equation}

Finally, the adaptive skill-enhanced prediction is obtained as
\begin{equation}
\label{eq:stage2_recognition}
\hat{y}_x
=
\begin{cases}
\mathcal{M}_{\Theta}\!\left(x,\mathbf{C}_x,G_x^{(1)}\right),
& H_x \leq \eta,\\
\mathcal{M}_{\Theta}\!\left(x,\mathbf{C}_x,G_x^{(2)}\right),
& H_x > \eta.
\end{cases}
\end{equation}

In this way, LIVE jointly implements candidate reduction and adaptive recognition guidance within the theoretical framework developed in Section~\ref{sec:problem_formulation}, thereby mitigating FGVR Laziness.

\section{Experiments}
\label{sec:experiments}

In this section, we conduct comprehensive experiments to evaluate the effectiveness and generalizability of the proposed LIVE.
The section is organized as follows.
In Section~\ref{sec:exp_settings}, we first introduce the experimental settings.
In Section~\ref{sec:exp_fgvc}, we compare LIVE with competitive methods on nine fine-grained image recognition benchmarks to verify its effectiveness on FGVR.
In Sections~\ref{sec:exp_action_recognition} and~\ref{sec:exp_object_recognition}, we further evaluate LIVE on video-based fine-grained action recognition and image recognition under both general and OOD settings, to investigate its generalizability across different recognition scenarios.
In Section~\ref{sec:exp_analysis}, we conduct ablation studies and additional analyses to further evaluate the proposed method.

\newcommand{\pcitag}{\raisebox{0.1ex}{\scalebox{0.72}{+PCI}}}

\begin{table*}[t]
\centering
\caption{
Comparisons with representative methods on nine fine-grained image recognition benchmarks.
Top-1 accuracy (\%) is reported, with the best and second-best results marked in \textbf{bold} and \underline{underlined}, respectively.
``$+\mathrm{PCI}$'' denotes augmenting the corresponding post-trained MLLM with the CLIP-based PCI module.
Results of the compared methods are taken from their original papers or reproduced using publicly available code.
}
\label{tab:main_image_results}
\scriptsize
\setlength{\tabcolsep}{3.0pt}
\renewcommand{\arraystretch}{1.12}
\newcommand{\gain}[1]{{\scriptsize #1}}
\resizebox{0.96\textwidth}{!}{
\begin{tabular}{@{}l c c c c c c c c c c c@{}}
\toprule[0.8pt]
Method & Source & Flower & CUB & Pet & Aircraft & Cars & Dogs
& DTD & Food & SUN & Avg. \\
\midrule

\rowcolor{black!8}\multicolumn{12}{@{}l}{\textbf{VLM-based Methods}} \\

CLIP RN50
& ICML 2021
& 61.18 & 46.08 & 83.59 & 15.60 & 53.88 & 54.56
& 38.72 & 78.62 & 58.71 & 54.55 \\

CLIP ViT-B/16
& ICML 2021
& 67.63 & 54.98 & 88.22 & 23.82 & 63.87 & 62.55
& 43.19 & 88.22 & 62.92 & 61.71 \\

SigLIP
& ICCV 2023
& 83.60 & 68.03 & 91.93 & 37.14 & \textbf{90.58} & \underline{74.98}
& 61.64 & 91.68 & 68.21 & 74.20 \\

DCLIP
& ICLR 2023
& 70.66 & 57.75 & 86.92 & 25.08 & 63.21 & 63.08
& 45.59 & 88.50 & 66.89 & 63.08 \\

CuPL
& ICCV 2023
& 73.31 & 56.77 & 91.20 & 27.93 & 66.12 & 66.78
& 54.26 & 88.85 & 68.99 & 66.02 \\

MPVR
& ECCV 2024
& 76.86 & 58.99 & 89.92 & 28.02 & 65.39 & 68.81
& 57.27 & 86.42 & 69.00 & 66.74 \\

ProAPO
& CVPR 2025
& 82.70 & 59.00 & \underline{92.70} & 28.30 & 67.70 & 67.56
& 55.80 & 89.50 & 68.90 & 68.02 \\

SCAN
& AAAI 2026
& 79.88 & 66.70 & 89.88 & 44.01 & 79.21 & 71.99
& 61.33 & 89.46 & 73.08 & 72.84 \\

\midrule
\rowcolor{black!8}\multicolumn{12}{@{}l}{\textbf{MLLM-based Methods}} \\

\multicolumn{12}{c}{\textit{General MLLMs}} \\

LLaVA-OneVision-7B
& arXiv 2024
& 33.69 & 19.05 & 65.49 & \underline{50.55} & 74.92 & 38.51
& 69.94 & 76.16 & 59.55 & 54.21 \\

InternVL3-8B
& arXiv 2025
& 38.60 & 24.71 & 65.63 & 20.52 & 31.30 & 33.08
& 61.91 & 70.00 & 58.69 & 44.94 \\

InternVL3.5-4B
& arXiv 2025
& 24.86 & 13.44 & 47.80 & 14.82 & 19.42 & 16.20
& 51.17 & 53.96 & 57.99 & 33.30 \\

InternVL3.5-8B
& arXiv 2025
& 27.97 & 15.51 & 55.73 & 15.51 & 22.53 & 21.09
& 57.07 & 57.63 & 59.24 & 36.92 \\

Qwen2.5-VL-7B
& arXiv 2025
& 71.65 & 40.66 & 86.15 & 45.30 & 60.70 & 61.65
& 62.07 & 83.94 & 62.57 & 63.85 \\

Qwen2.5-VL-32B
& arXiv 2025
& 78.71 & 51.65 & 86.67 & 44.40 & 74.50 & 62.82
& 61.96 & 81.90 & 68.02 & 67.85 \\

Qwen3-VL-8B
& arXiv 2025
& 80.79 & 44.92 & 87.08 & 38.58 & 69.44 & 59.34
& 67.97 & 87.37 & 70.74 & 67.36 \\

\midrule
\multicolumn{12}{c}{\textit{Post-trained MLLMs}} \\

Finedefics
& ICLR 2025
& 47.53 & 30.51 & 71.70 & 28.68 & 52.31 & 43.89
& 19.41 &  53.30 & 56.36 & 44.85 \\

Finedefics \pcitag
& -
& 77.67 & 50.41 & 89.75 & 39.21 & 80.16  & 66.59
& 48.46 & 87.45 & 66.82 & 67.39 \\

Fine-R1-7B
& ICLR 2026
& 71.40 & 46.94 & 85.33 & 46.08 & 57.49 & 61.09
& 49.78 & 78.77 & 59.11 & 61.78 \\

Fine-R1-7B \pcitag
& -
& 79.51 & 61.98 & 91.36 & 40.56 & 68.13 & 70.51
& 58.51 & 87.79 & 71.64 & 70.00 \\

DiVE-k
& ICLR 2026
& 71.90 & 55.37 & 87.84 & 46.56 & 59.87 & 69.91
& 56.44 & 79.79 & 60.79 & 65.39 \\

DiVE-k \pcitag
& -
& 85.18 & 64.56 & 91.63 & 49.38 & 82.35 & 74.11
& 63.98 & 87.32 & 70.80 & 74.37 \\

\midrule

\rowcolor{black!8}\multicolumn{12}{c}{Qwen2.5-VL-7B} \\

LIVE$_c$
& 
& 86.19 & 64.98 & 92.01 & 46.02 & 77.81 & 72.84
& 64.78 & 90.18 & 75.27 & 74.45 \\

$\Delta$ vs. Qwen2.5-VL-7B
& 
& \gain{+14.54} & \gain{+24.32} & \gain{+5.86}
& \gain{+0.72} & \gain{+17.11} & \gain{+11.19}
& \gain{+2.71} & \gain{+6.24} & \gain{+12.70}
& \gain{+10.60} \\

LIVE$_s$
& 
& 90.09 & \textbf{70.17} & 92.28 & \textbf{53.28}
& \underline{88.78} & \textbf{76.80}
& \underline{70.15} & \underline{91.73} & \textbf{76.02}
& \underline{78.81} \\

$\Delta$ vs. Qwen2.5-VL-7B
& 
& \gain{+18.44} & \gain{+29.51} & \gain{+6.13}
& \gain{+7.98} & \gain{+28.08} & \gain{+15.15}
& \gain{+8.08} & \gain{+7.79} & \gain{+13.45}
& \gain{+14.96} \\

\addlinespace[2pt]

\rowcolor{black!8}\multicolumn{12}{c}{Qwen3-VL-8B} \\

LIVE$_c$
& 
& \underline{90.29} & 65.29 & 92.25 & 43.11 & 80.47 & 70.95
& 69.04 & 91.38 & 74.76 & 75.28 \\

$\Delta$ vs. Qwen3-VL-8B
& 
& \gain{+9.50} & \gain{+20.37} & \gain{+5.17}
& \gain{+4.53} & \gain{+11.03} & \gain{+11.61}
& \gain{+1.07} & \gain{+4.01} & \gain{+4.02}
& \gain{+7.92} \\

LIVE$_s$
& 
& \textbf{93.59} & \underline{70.07} & \textbf{92.96}
& 50.52 & 88.43 & 74.84
& \textbf{72.34} & \textbf{92.37} & \underline{75.91}
& \textbf{79.00} \\

$\Delta$ vs. Qwen3-VL-8B
& 
& \gain{+12.80} & \gain{+25.15} & \gain{+5.88}
& \gain{+11.94} & \gain{+18.99} & \gain{+15.50}
& \gain{+4.37} & \gain{+5.00} & \gain{+5.17}
& \gain{+11.64} \\

\bottomrule[0.8pt]
\end{tabular}
}
\end{table*}

\subsection{Experimental Settings}
\label{sec:exp_settings}

\textbf{Datasets.}
We conduct experiments on 13 widely used benchmarks covering three representative recognition scenarios.
For fine-grained image recognition, we adopt nine FGVR benchmarks, including Flowers102~\cite{flower}, CUB200~\cite{cub}, Pet37~\cite{pet}, FGVC-Aircraft~\cite{aircraft}, Cars196~\cite{car}, Dogs120~\cite{dog}, DTD~\cite{dtd}, Food101~\cite{food}, and SUN397~\cite{sun}.
These datasets cover diverse domains such as plants, animals, vehicles, textures, food, and scenes.
For video-based fine-grained action recognition, we employ UCF101~\cite{ucf101} and HMDB51~\cite{hmdb51}.
For general and OOD image recognition, we further evaluate LIVE on ImageNet-1K~\cite{imagenet-1k} and ImageNet-Sketch~\cite{imagenet-sketch}, respectively.
Additional details of these datasets are provided in the Supplementary Material.

\textbf{Implementation Details.}
As a training-free framework, LIVE can be readily integrated with different MLLMs.
For comprehensive evaluations, we implement LIVE with Qwen2.5-VL-7B~\cite{qwen2.5-vl} and Qwen3-VL-8B~\cite{qwen3-vl}.
Meanwhile, for visual relevance computation in PCI (Eq.~\ref{eq:visual_relevance}), we employ CLIP ViT-B/16~\cite{clip} and SigLIP-base-patch16-224~\cite{siglip} as alternative VLM backbones, with the corresponding variants denoted as $\mathrm{LIVE}_c$ and $\mathrm{LIVE}_s$, respectively.
Accordingly, the threshold $\gamma$ in Eq.~\ref{eq:low_conf_candidate_list} is set to $0.035$ and $2.5$ for CLIP and SigLIP, respectively.
Moreover, the number of seeds $B$, the threshold $\rho$, and the candidate budget $K$ in PCI are set to $5$, $0.6$, and $20$, respectively.
For SCR, the entropy threshold $\eta$ is set to $0.1$ and $0.01$ for Qwen2.5-VL-7B and Qwen3-VL-8B, respectively.
Notably, the above hyperparameters are fixed across all datasets.
More implementation details and prompt templates are provided in the Supplementary Material.

\textbf{Evaluation Protocol.}
We follow the standard data split of each dataset and report Top-1 accuracy for all experiments.
Unless otherwise specified, we adopt a 4-shot reference setting (\textit{i.e.}, $N=4$), where four samples per category are selected from the training split to construct the affinity graph in PCI and Category Skills in CSC.
All evaluations are performed on the corresponding test split.
For video datasets, following CLIP~\cite{clip}, we use the middle frame of each video as input to the VLM while providing the complete video to the MLLM.

\subsection{Comparisons on Fine-Grained Image Recognition}
\label{sec:exp_fgvc}

In this section, we evaluate LIVE on nine widely used fine-grained image recognition benchmarks and compare it with representative methods.
Specifically, the compared methods are grouped into two categories:
i) VLM-based methods, including CLIP~\cite{clip}, SigLIP~\cite{siglip}, DCLIP~\cite{dclip}, CuPL~\cite{cupl}, MPVR~\cite{mpvr}, ProAPO~\cite{proapo}, and SCAN~\cite{scan}, covering both vanilla image-text matching and LLM-enhanced recognition;
ii) MLLM-based methods, including general MLLMs from the LLaVA-OneVision~\cite{llava-onevision}, InternVL~\cite{internvl3,internvl3-5}, and Qwen-VL~\cite{qwen2.5-vl,qwen3-vl} families, together with FGVR-oriented post-trained models, such as Finedefics~\cite{finedefics}, Fine-R1~\cite{fine-r1}, and DiVE-k~\cite{dive-k}.

For fair comparisons, we use the results reported in the original papers whenever available and reproduce the remaining baselines using their official implementations and recommended settings.
Unless otherwise specified, all baselines are evaluated over the full category list (\textit{i.e.}, $\mathbf{C}$).
For Finedefics, as overly long candidate lists impair instruction following, we restrict each list to at most $50$ categories, including the ground-truth category and up to $49$ randomly sampled alternatives.
Moreover, since our approach to alleviating FGVR Laziness is orthogonal to FGVR-oriented post-training, we further endow Finedefics, Fine-R1, and DiVE-k with the PCI module.
The corresponding variants are denoted by ``$+\mathrm{PCI}$''.

From the results in Table~\ref{tab:main_image_results}, one could have the following observations and conclusions:
\textbf{i) General-purpose MLLMs still exhibit limited FGVR performance despite strong underlying visual backbones.}
For example, LLaVA-OneVision-7B, which builds upon a SigLIP-family vision encoder, achieves only 54.21\% average accuracy, substantially lower than the 74.20\% of SigLIP.
Even the strongest general MLLM, Qwen2.5-VL-32B, reaches only 67.85\%, remaining 6.35 points below SigLIP.
Such performance gaps highlight the necessity of further eliciting the FGVR capability of general-purpose MLLMs.
\textbf{ii) FGVR-specific post-training brings notable gains on post-training domains, yet the acquired capability generalizes poorly to unseen domains.}
For example, Fine-R1 and DiVE-k, both built upon Qwen2.5-VL-7B, achieve clear improvements on the domains involved in post-training, yet consistently underperform the base model on the unseen DTD, Food-101, and SUN397 datasets.
\textbf{iii) FGVR-oriented post-trained MLLMs consistently benefit from PCI.}
Specifically, Finedefics, Fine-R1, and DiVE-k achieve further average gains of 22.54, 8.22, and 8.98 points with PCI, respectively.
These consistent improvements indicate that reducing category competition remains beneficial even after FGVR-oriented post-training.
\textbf{iv) LIVE consistently achieves substantial FGVR gains across different MLLM backbones and benchmarks.}
Specifically, $\mathrm{LIVE}_s$ improves Qwen2.5-VL-7B and Qwen3-VL-8B by 14.96 and 11.64 points on average, respectively.
These results further highlight inference-time FGVR Laziness mitigation as a promising avenue for improving MLLM-based FGVR.

\subsection{Comparisons on Fine-Grained Action Recognition}
\label{sec:exp_action_recognition}

\begin{table}[t]
\centering
\caption{
Comparisons with representative methods on human action recognition.
Top-1 accuracy (\%) is reported on UCF101 and HMDB51 datasets, with the best results marked in \textbf{bold}.
}
\label{tab:action_results}
\scriptsize
\setlength{\tabcolsep}{3.0pt}
\renewcommand{\arraystretch}{1.12}
\resizebox{\columnwidth}{!}{
\begin{tabular}{@{}l c c c c@{}}
\toprule[0.8pt]
Method & Source & UCF101 & HMDB51 & Avg. \\

\midrule
\rowcolor{black!8}\multicolumn{5}{c}{\textit{General VLMs}} \\
CLIP RN50 & ICML 2021 & 62.72 & 37.77 & 50.25 \\
CLIP ViT-B/16 & ICML 2021 & 67.88 & 47.25 & 57.57 \\
SigLIP & ICCV 2023 & 74.86 & 46.14 & 60.50 \\
SigLIP2 & arXiv 2025 & 73.24 & 45.29 & 59.27 \\

\midrule
\rowcolor{black!8}\multicolumn{5}{c}{\textit{General MLLMs}} \\
Qwen2.5-VL-7B & arXiv 2025 & 79.54 & 56.93 & 68.24 \\
Qwen2.5-VL-32B & arXiv 2025 & 75.79 & 54.77 & 65.28 \\
Qwen3-VL-8B & arXiv 2025 & 84.54 & 63.92 & 74.23 \\

\midrule
\rowcolor{black!8}\multicolumn{5}{c}{\textit{Action Recognition Methods}} \\
DTS-TPT & IJCAI 2024 & 79.40 & 54.40 & 66.90 \\
TEST-V & IJCAI 2025 & 78.30 & 53.00 & 65.65 \\
Hierarchical Adapter & IJCAI 2025 & 77.60 & 54.90 & 66.25 \\
Visual-Textual Refinement & ICCVW 2025 & 79.60 & 57.10 & 68.35 \\
CONDA & CVPR 2026 & 81.50 & 55.20 & 68.35 \\

\midrule
\rowcolor{black!8}\multicolumn{5}{c}{\textit{Our Method}} \\
LIVE$_s$ (Qwen2.5-VL-7B) & -- & 84.37 & 56.79 & 70.58 \\
LIVE$_s$ (Qwen3-VL-8B) & -- & \textbf{86.33} & \textbf{64.77} & \textbf{75.55} \\
\bottomrule[0.8pt]
\end{tabular}
}
\end{table}

In this section, we conduct experiments on human action recognition, which requires distinguishing subtle action differences from dynamic visual cues.
Besides general VLM and MLLM baselines, we further compare LIVE with representative action recognition methods, including DTS-TPT~\cite{DTS-TPT}, TEST-V~\cite{test-v}, Hierarchical Adapter~\cite{hierarchical-adapter}, Visual-Textual Refinement~\cite{visual-text-refine}, and CONDA~\cite{conda}. 
These methods cover different paradigms for enhancing VLM-based action recognition, including test-time adaptation, representation refinement, and temporal modeling.
As shown in Table~\ref{tab:action_results}, LIVE improves the average action recognition performance across different MLLM backbones.
In particular, $\mathrm{LIVE}_s$ brings a $1.32$-point average gain over Qwen3-VL-8B and surpasses CONDA by $7.20$ points.
These results further demonstrate the generalizability of LIVE to video-based fine-grained recognition.

\subsection{Comparisons on General and OOD Image Recognition}
\label{sec:exp_object_recognition}

\begin{table}[t]
\centering
\caption{
Comparisons with representative methods on general and OOD image recognition.
Top-1 accuracy (\%) is reported on ImageNet-1K and ImageNet-Sketch, abbreviated as ImageNet-S in the table, with the best results marked in \textbf{bold}.
``$+\mathrm{PCI}$'' denotes augmenting the corresponding post-trained MLLM with the SigLIP-based PCI module.
}
\label{tab:object_results}
\scriptsize
\setlength{\tabcolsep}{3.0pt}
\renewcommand{\arraystretch}{1.12}
\resizebox{\columnwidth}{!}{
\begin{tabular}{@{}l c c c c@{}}
\toprule[0.8pt]
Method & Source & ImageNet-1K & ImageNet-S & Avg. \\
\midrule
\rowcolor{black!8}\multicolumn{5}{c}{\textit{General VLMs}} \\
CLIP RN50 & ICML 2021 & 55.50 & 32.35 & 43.93 \\
CLIP ViT-B/16 & ICML 2021 & 63.84 & 44.98 & 54.41 \\
SigLIP & ICCV 2023 & 71.58 & 64.80 & 68.19 \\
SigLIP2 & arXiv 2025 & 73.08 & 65.10 & 69.09 \\
\midrule
\rowcolor{black!8}\multicolumn{5}{c}{\textit{MLLM-based Methods}} \\
Qwen2.5-VL-7B & arXiv 2025 & 61.01 & 56.00 & 58.51 \\
Qwen2.5-VL-32B & arXiv 2025 & 71.51 & 59.26 & 65.39 \\
Qwen3-VL-8B & arXiv 2025 & 70.68 & 63.01 & 66.85 \\
\midrule
Finedefics & ICLR 2025 & 55.86 & 56.64 & 56.25 \\
Finedefics \pcitag & - & 72.59 & 65.84 & 69.22 \\
Fine-R1-7B & ICLR 2026 & 61.27 & 48.95 & 55.11 \\
Fine-R1-7B \pcitag & - & 75.30 & 66.52 & 70.91 \\
DiVE-k & ICLR 2026 & 56.21 & 45.37 & 50.79 \\
DiVE-k \pcitag & - & 76.44 & 67.84 & 72.14 \\
\midrule
\rowcolor{black!8}\multicolumn{5}{c}{\textbf{\textit{Our Method}}} \\
LIVE$_s$ (Qwen2.5-VL-7B) & -- & 77.32 & 68.75 & 73.04 \\
LIVE$_s$ (Qwen3-VL-8B) & -- & \textbf{77.66} & \textbf{69.90} & \textbf{73.78} \\
\bottomrule[0.8pt]
\end{tabular}
}
\end{table}

In this section, we further evaluate LIVE on general and out-of-distribution (OOD) image recognition using ImageNet-1K and ImageNet-Sketch, respectively.
ImageNet-1K covers a large and diverse category list containing both coarse-grained object categories and fine-grained subordinate classes, while ImageNet-Sketch introduces a sketch-style distribution shift.
The compared methods include general VLMs, general MLLMs, and FGVR-oriented post-trained MLLMs.
Following the previous experiment, we further endow the post-trained MLLMs with our SigLIP-based PCI module, with the corresponding variants denoted by ``$+\mathrm{PCI}$''.

From the results in Table~\ref{tab:object_results}, one could have the following observations:
\textbf{i) FGVR-oriented post-training may compromise the generalization performance of MLLMs, particularly under OOD shifts.}
For example, Fine-R1 and DiVE-k, both built upon Qwen2.5-VL-7B, show average drops of 3.40 and 7.72 points relative to the base model, with larger decreases of 7.05 and 10.63 points on ImageNet-Sketch.
\textbf{ii) FGVR-oriented post-trained MLLMs consistently benefit from PCI on general and OOD recognition.}
Specifically, Finedefics, Fine-R1, and DiVE-k obtain average gains of 12.97, 15.80, and 21.35 points with PCI, respectively, indicating that reducing unnecessary category competition remains beneficial beyond FGVR.
\textbf{iii) LIVE achieves strong performance on both general and OOD image recognition.}
In particular, $\mathrm{LIVE}_s$ with Qwen3-VL-8B achieves the best average accuracy of $73.78\%$, improving its base model by $6.93$ points and outperforming Finedefics by $17.53$ points, highlighting the broader applicability of inference-time FGVR Laziness mitigation.

\subsection{Ablation Studies and Analytic Experiments}
\label{sec:exp_analysis}

In this section, we conduct a series of ablation studies and analytic experiments to further investigate LIVE.
Specifically, we first examine the contribution of each component and its performance across different model scales, followed by detailed analyses of the PCI, CSC, and SCR modules.
Unless otherwise stated, all experiments are conducted on the Flower and CUB datasets using LIVE$_c$, \textit{i.e.}, CLIP ViT-B/16 for PCI and Qwen2.5-VL-7B for CSC and SCR, respectively.

\subsubsection{Ablation Study}
\begin{figure}[t]
    \centering
    \includegraphics[width=0.99\linewidth]{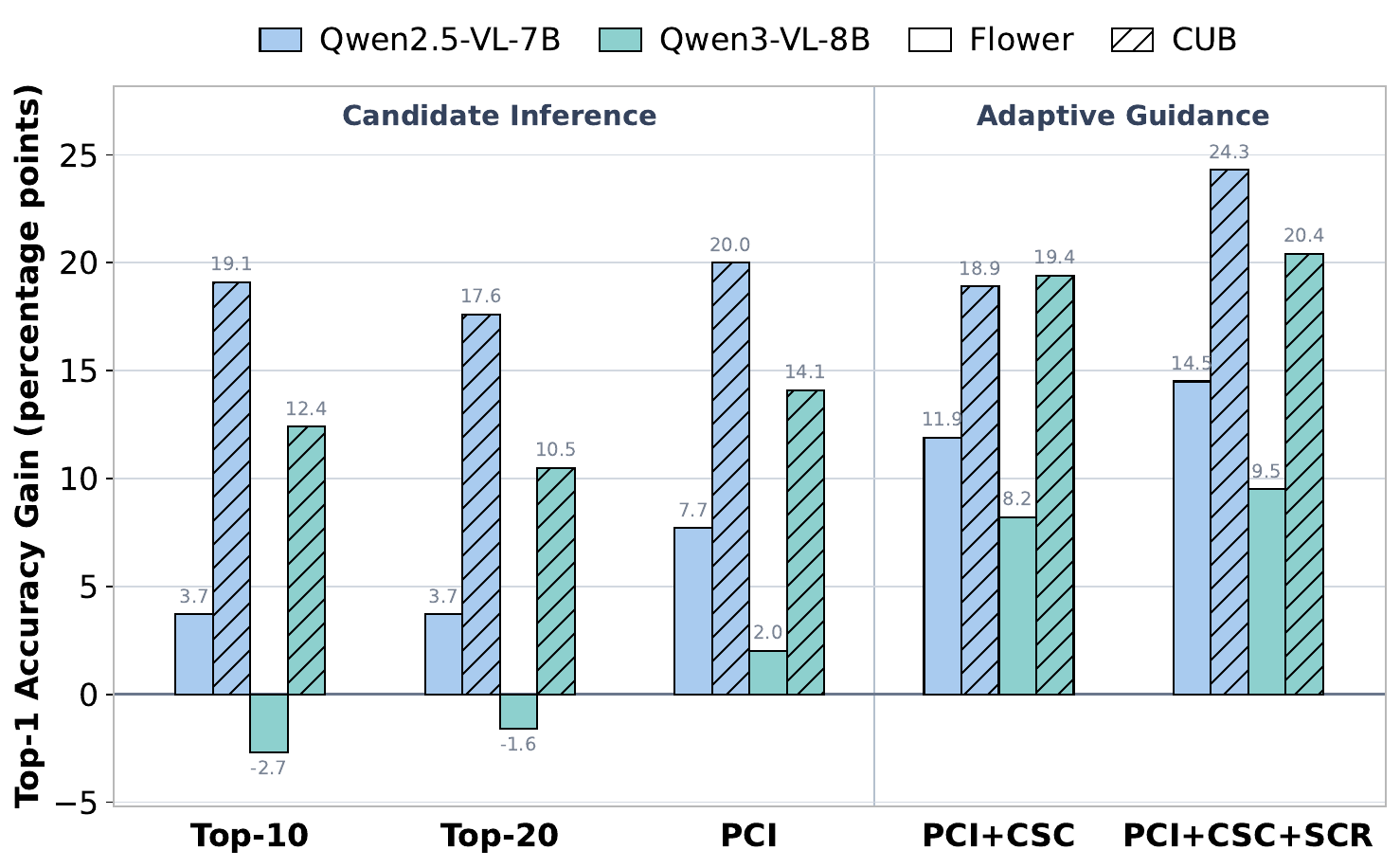}
\caption{
Ablation study of LIVE$_c$ on Flower and CUB datasets.
Results are reported as Top-1 accuracy gains over the corresponding MLLM baseline.
Here, Top-10 and Top-20 use CLIP-ranked candidates for recognition, while PCI uses the inferred plausible candidates.
PCI+CSC further introduces skill headers, and PCI+CSC+SCR denotes the complete $\mathrm{LIVE}_c$.
}
    \label{fig:ablation}
\end{figure}

\begin{figure}[t]
    \centering
    \includegraphics[width=0.99\linewidth]{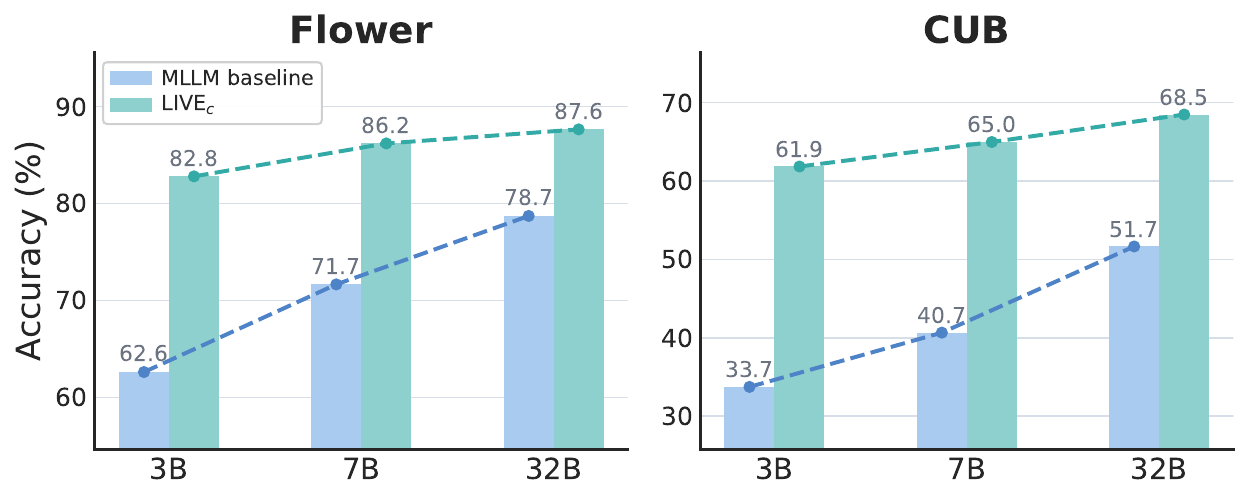}
    \caption{
    Scaling analysis of LIVE$_c$ across different Qwen2.5-VL model sizes on the Flower and CUB datasets.
    }
    \label{fig:scaling_law}
\end{figure}

To examine the contribution of each component in LIVE, we construct a series of variants and use the base MLLM over the full category list (\textit{i.e.}, $\mathbf{C}$) as the baseline.
Specifically, as shown in Fig.~\ref{fig:ablation},
i) The Top-10 and Top-20 settings perform recognition over fixed-size candidate lists selected from the top-$k$ CLIP similarity rankings;
ii) The PCI setting performs recognition over the inferred  candidate list (\textit{i.e.}, $\mathbf{C}_{x}$);
iii) PCI+CSC further introduces the skill headers (\textit{i.e.}, name and description) for these candidates;
iv) PCI+CSC+SCR, \textit{i.e.}, the complete $\mathrm{LIVE}_c$, additionally invokes the detailed skill bodies for uncertain samples through SCR.

From the results in Fig.~\ref{fig:ablation}, one could have the following observations and conclusions.
\textbf{i) A compact yet reliable candidate list is beneficial for FGVR.}
The Top-10 and Top-20 settings improve recognition in most cases, but degrade the performance of Qwen3-VL-8B on Flower.
In contrast, PCI consistently outperforms both fixed Top-$k$ variants across different datasets and MLLM backbones.
A possible reason is that fixed Top-$k$ selection struggles to balance ground-truth coverage and candidate competition. 
In other words, aggressive truncation might exclude the ground truth, whereas larger candidate lists retain more competing categories, making FGVR Laziness harder to alleviate.
\textbf{ii) Category-specific guidance further facilitates fine-grained discrimination.}
Building on PCI, introducing skill headers generally yields further gains, including an improvement of $6.2$ percentage points for Qwen3-VL-8B on Flower.
This suggests that even lightweight category-specific guidance can facilitate the association between visual cues and category concepts, thereby alleviating FGVR Laziness.
\textbf{iii) Selectively invoking detailed guidance further improves recognition for uncertain queries.}
With the further incorporation of SCR, the complete $\mathrm{LIVE}_c$ achieves an accuracy gain of $24.3$ percentage points on CUB with Qwen2.5-VL-7B.
These results suggest that selectively providing detailed guidance for uncertain queries can further strengthen the association between visual cues and category concepts, thereby mitigating FGVR Laziness.

\begin{figure*}[t]
\centering
{
    \includegraphics[width=0.235\textwidth]{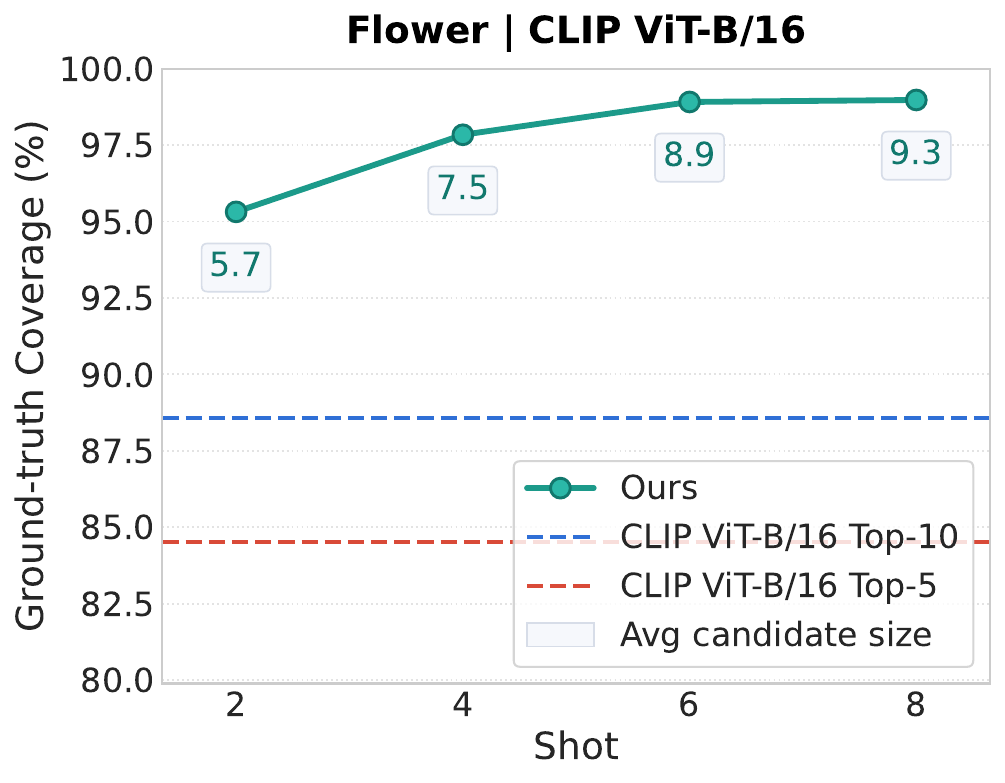}
}
\hfill
{
    \includegraphics[width=0.235\textwidth]{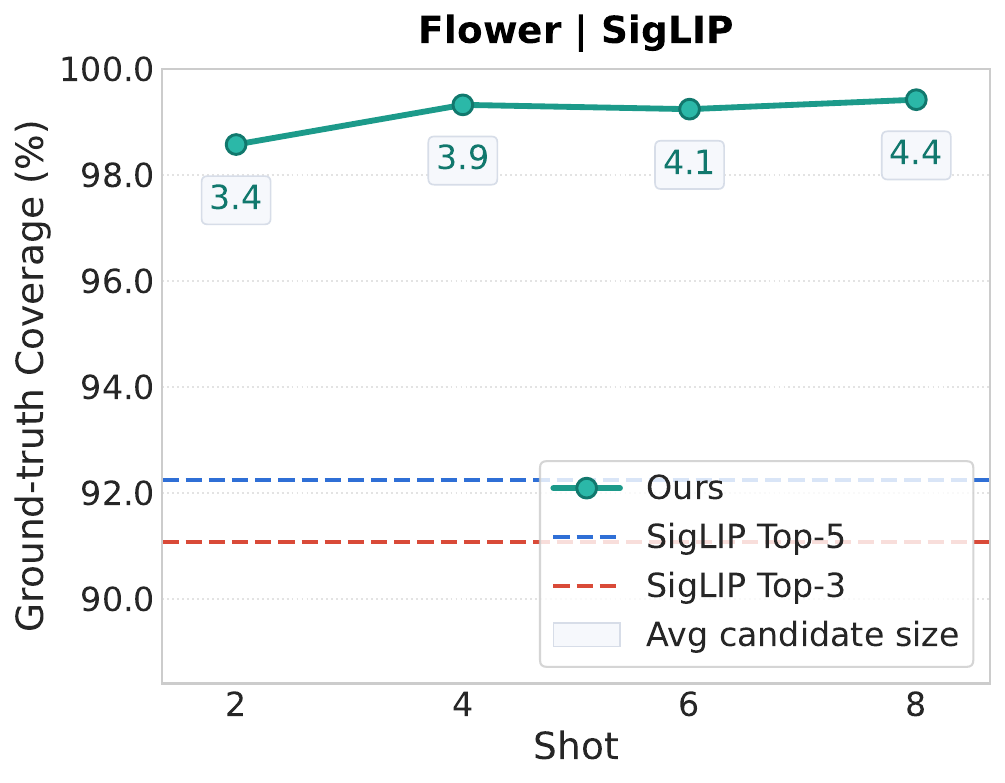}
}
\hfill
{
    \includegraphics[width=0.235\textwidth]{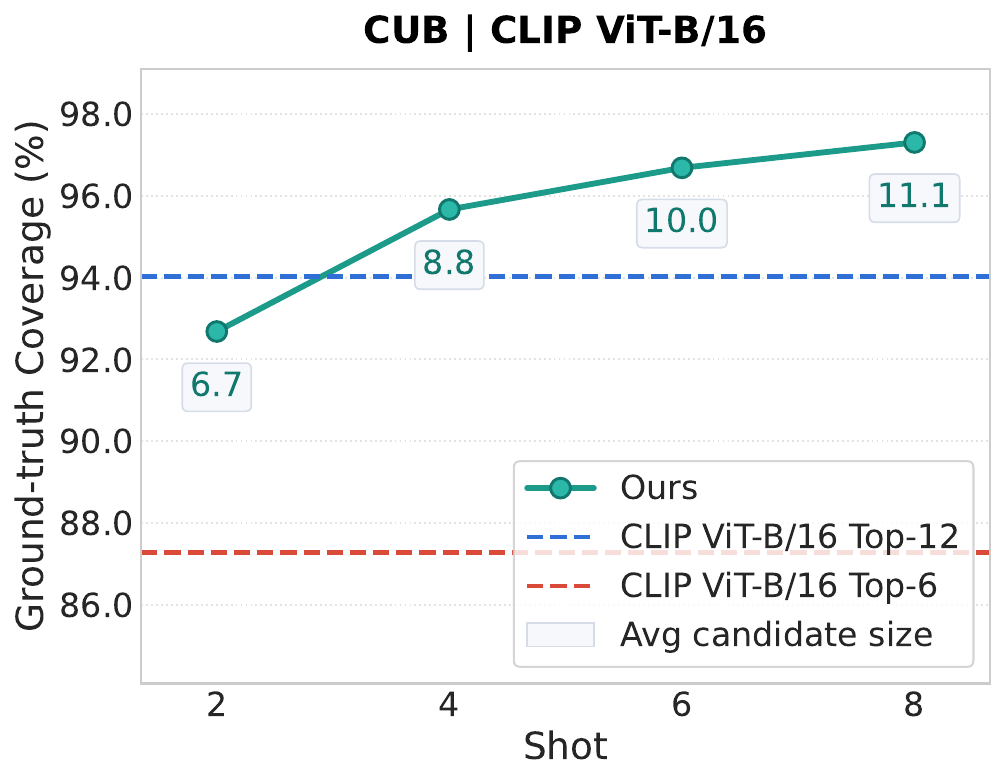}
}
\hfill
{
    \includegraphics[width=0.235\textwidth]{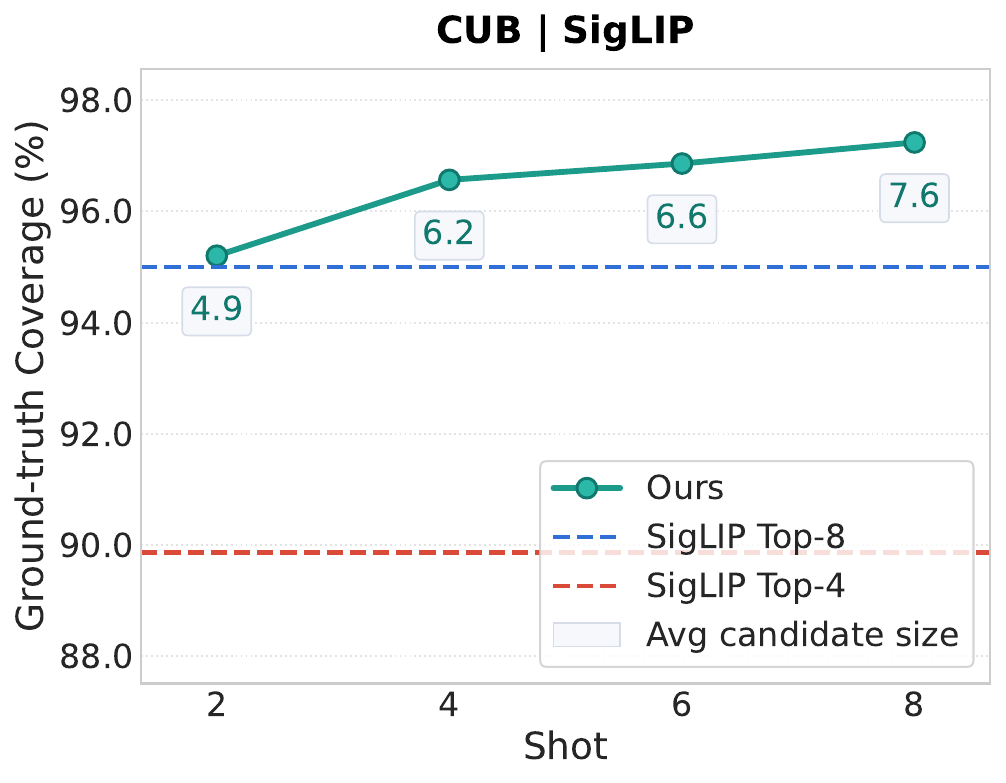}
}
\caption{
Quantitative analysis of PCI with different VLM backbones and reference-shot settings.
The solid curves report the ground-truth coverage, while the annotated values denote the corresponding average size of the inferred candidate list $\mathbf{C}_{x}$.
The dashed lines report fixed Top-$k$ baselines, with $k$ chosen according to the lower and upper bounds of the candidate-list sizes produced by PCI.
}

\label{fig:cci_analysis}
\end{figure*}

\subsubsection{Scaling Analysis}

To investigate the effectiveness of LIVE across different model scales, we further evaluate LIVE$_c$ with Qwen2.5-VL models ranging from 3B to 32B parameters.
As shown in Fig.~\ref{fig:scaling_law}, FGVR performance steadily improves with model size, while $\mathrm{LIVE}_c$ provides consistent gains across all scales.
For example, on CUB, $\mathrm{LIVE}_c$ improves Qwen2.5-VL-3B from 33.7\% to 61.9\%.
Notably, the 3B model equipped with $\mathrm{LIVE}_c$ already outperforms the base 32B model on both Flower and CUB.
These results demonstrate that mitigating FGVR Laziness could complement model scaling in improving FGVR performance.

\subsubsection{Analytic Study on PCI}

Theorem~\ref{thm:risk} in Section~\ref{sec:problem_formulation} shows that candidate reduction is a lower-risk strategy for mitigating FGVR Laziness when the probability of retaining the ground-truth category $p_{\mathrm{ret}}$ satisfies $p_{\mathrm{ret}}>\frac{1}{2}+\frac{1}{2M}$.
To verify whether PCI satisfies this condition in practice, we report its ground-truth coverage and average candidate-list size under different VLM backbones and reference-shot settings, with fixed Top-$k$ selection as baselines.
As shown in Fig.~\ref{fig:cci_analysis}, PCI consistently maintains high ground-truth coverage with relatively small candidate sets. 
For example, on Flower with CLIP, the 4-shot setting retains only 7.5 candidates on average while achieving over 97\% coverage, whereas fixed Top-10 achieves less than 90\%. 
Moreover, increasing the number of reference samples further improves coverage with only a modest increase in candidate-list size.
These results verify that PCI satisfies the theoretical requirement for low-risk candidate reduction while maintaining a compact candidate list, thereby reducing category competition to mitigate FGVR Laziness.

\begin{figure}[t]
    \centering
    \includegraphics[width=0.99\linewidth]{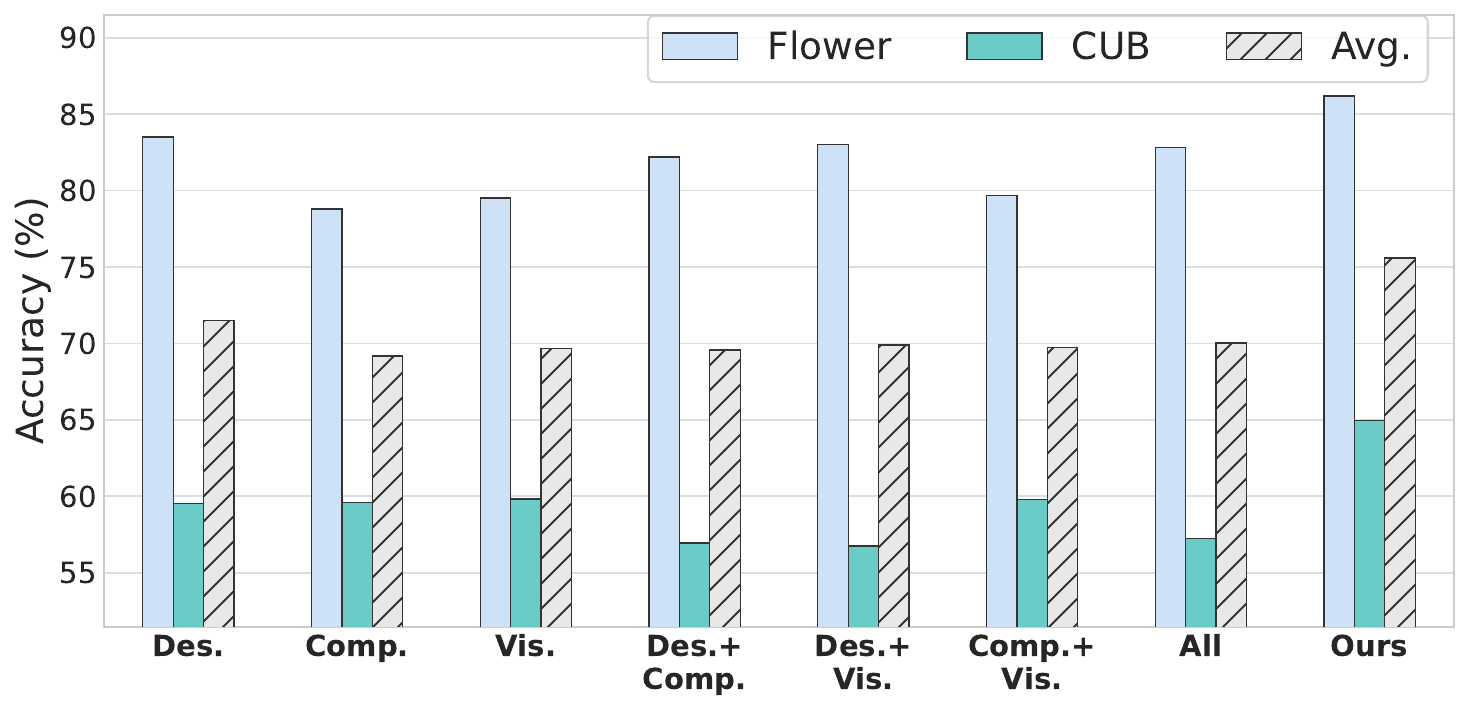}
    \caption{
Analytic study on Category Skill components and invocation strategy.
``Des.'', ``Vis.'', and ``Comp.'' denote category description, visual checklist, and comparative information, respectively.
``All'' introduces all components simultaneously. ``Ours'' invokes detailed skill bodies on demand through SCR.
}
    \label{fig:knowledge_type}
\end{figure}

\subsubsection{Analytic Study on CSC}

We further investigate how the recognition guidance $G$ could be constructed using different components of Category Skills.
As shown in Fig.~\ref{fig:knowledge_type}, different skill components and their combinations exhibit varying effectiveness, and simply introducing more guidance does not necessarily improve recognition.
In particular, simultaneously introducing all components remains inferior to selectively invoking detailed guidance according to the current recognition state.
Consistent with this observation, the complete $\mathrm{LIVE}_c$ achieves the best performance on both datasets.
These results demonstrate that effective skill utilization depends not only on informative skill content but also on adaptive invocation, which enables more targeted recognition guidance for mitigating FGVR Laziness.

\subsubsection{Analytic Study on SCR}

\begin{figure}[t]
    \centering
    \includegraphics[width=0.99\linewidth]{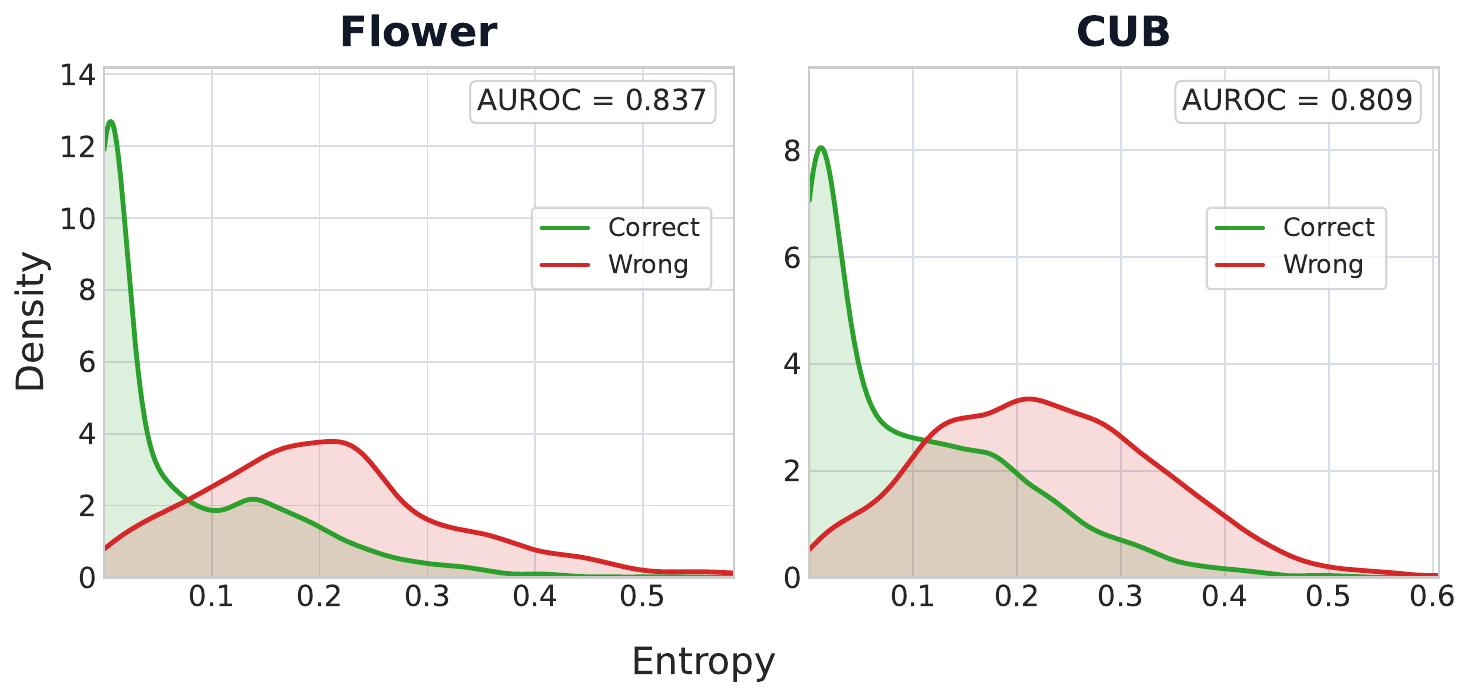}
    \caption{
Entropy distributions of correct and incorrect predictions.
Higher entropy indicates greater prediction uncertainty.
}
    \label{fig:answer_entropy}
\end{figure}

\begin{table}[t]
\centering
\caption{
Analytic study on SCR.
The upper part evaluates entropy-based uncertainty identification, while the lower part compares detailed skill invocation for all candidates and confused categories on high-entropy samples.
}
\label{tab:scr_analysis}

\small
\renewcommand{\arraystretch}{1.15}
\setlength{\tabcolsep}{3.5pt}
\begin{tabular*}{\columnwidth}
{@{\extracolsep{\fill}}lccc@{}}
\toprule
Metric & Flower & CUB & Avg. \\
\midrule
\rowcolor{black!8}
\multicolumn{4}{c}{\textit{Entropy-based Uncertainty Identification}} \\
Low-entropy Acc. (\%)
& 95.46 & 87.64 & 91.55 \\
High-entropy Acc. (\%)
& 58.59 & 37.40 & 48.00 \\
High-entropy Ratio (\%)
& 32.44 & 55.97 & 44.21 \\

\midrule
\rowcolor{black!8}
\multicolumn{4}{c}{\textit{Detailed Skill Invocation for High-entropy Samples}} \\
Avg. Size of All Candidates
& 9.94 & 10.61 & 10.28 \\
Avg. Size of Confused Candidates
& 5.62 & 8.35 & 6.99 \\

Skills for All Candidates Acc. (\%)
& 64.51 & 46.75 & 55.63 \\
Skills for Confused Categories Acc. (\%)
& \textbf{66.86} & \textbf{47.14} & \textbf{57.00} \\
\bottomrule
\end{tabular*}
\end{table}

SCR determines when detailed skill bodies should be invoked and which candidate categories require such guidance.
We first examine whether answer entropy can reliably reflect prediction uncertainty.
As shown in Fig.~\ref{fig:answer_entropy}, incorrect predictions generally exhibit higher entropy than correct ones.
Accordingly, we use a fixed entropy threshold to identify uncertain predictions.
As shown in Table~\ref{tab:scr_analysis}, low-entropy samples achieve 91.55\% accuracy on average, whereas high-entropy samples achieve only 48.00\%, supporting entropy as a reliable criterion for uncertainty identification. 
For these uncertain predictions, SCR further identifies confused categories through prefix matching and invokes detailed skill bodies only for these categories.
As shown in Table~\ref{tab:scr_analysis}, prefix matching reduces the average number of categories for detailed skill invocation from 10.28 to 6.99, while targeted invocation improves the average accuracy from 55.63\% to 57.00\% compared with invoking detailed guidance for all candidates.
These results demonstrate that prediction uncertainty and category confusion provide effective signals for on-demand skill invocation, enabling more targeted recognition guidance to mitigate FGVR Laziness.

\subsubsection{Parameter Analysis}
\label{sec:parameter_analysis}

To investigate the robustness of LIVE, we conduct experiments on Flower and CUB using $\mathrm{LIVE}_c$, varying one hyperparameter at a time while keeping the others at their default values.

\begin{figure}[t]
    \centering
    \includegraphics[width=0.48\textwidth]{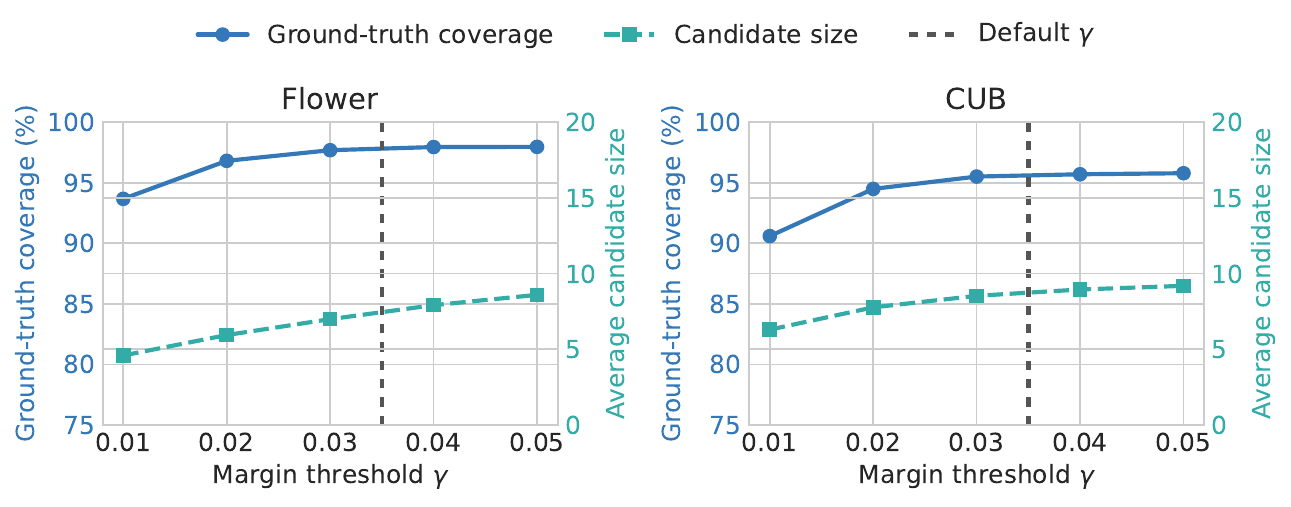}
 
    \includegraphics[width=0.48\textwidth]{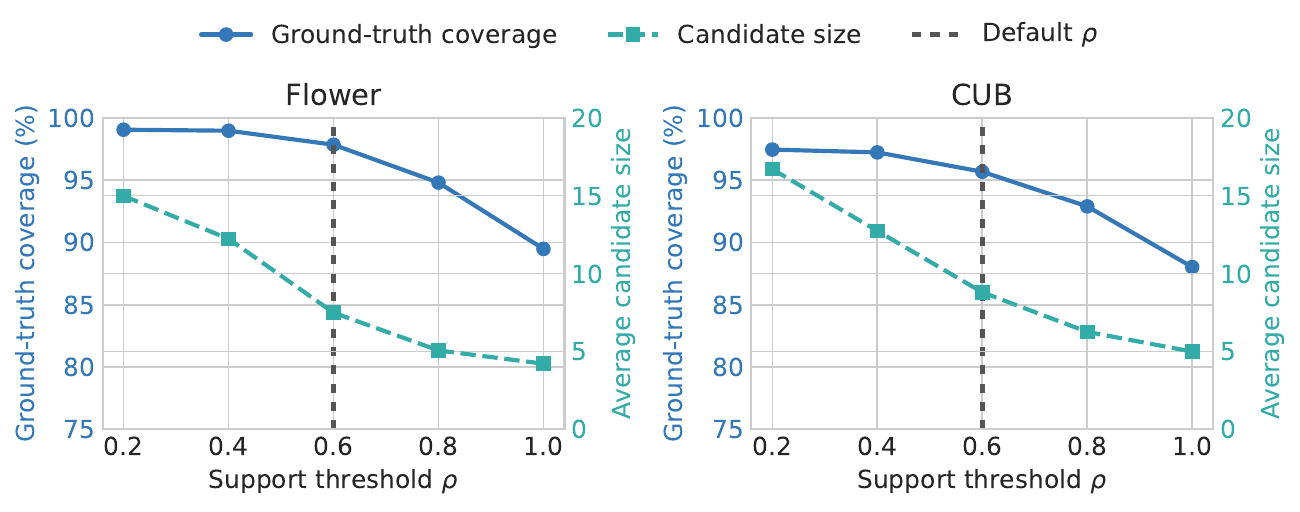}

    \includegraphics[width=0.48\textwidth]{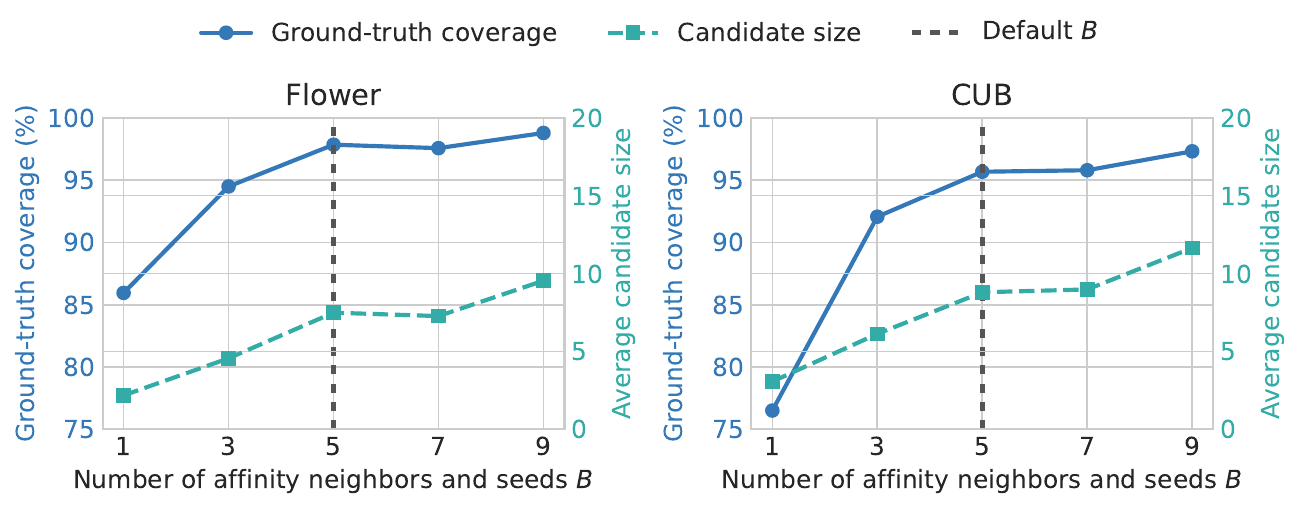}
  
    \caption{
    Parameter analysis of PCI on Flower and CUB.
    The three rows examine the ranking-reliability threshold $\gamma$, support threshold $\rho$, and seed number $B$, respectively.
    The blue and green curves report ground-truth coverage and average candidate size, while the dotted lines indicate the default parameter settings.
    }
    \label{fig:pci_parameter}
\end{figure}

\textbf{Parameter Analysis of PCI.}
As shown in Fig.~\ref{fig:pci_parameter}, one could have the following observations.
i) The ranking-reliability threshold $\gamma$ in Eq.~\ref{eq:low_conf_candidate_list} determines the boundary between different candidate recall strategies.
As $\gamma$ increases, more queries fall into the low-reliability interval and are handled by the corresponding recall strategy, leading to higher ground-truth coverage with a moderate increase in candidate size.
Notably, PCI consistently maintains over $90\%$ coverage with fewer than $10$ candidates on average across all evaluated values of $\gamma$, demonstrating its robustness to the ranking-reliability threshold;
ii) The support threshold $\rho$ in Eq.~\ref{eq:multi_seed_support} controls the selection of potentially confused categories.
A smaller $\rho$ adopts a looser matching criterion and recalls more candidates, whereas a larger $\rho$ yields a more compact candidate set but may sacrifice ground-truth coverage.
The default $\rho=0.6$ therefore achieves a favorable trade-off between ground-truth coverage and candidate compactness;
iii) The seed number $B$ in Eq.~\ref{eq:multi_seed_support} affects how much confusion information is considered during candidate inference.
Increasing $B$ introduces more seed categories and richer affinity information, thereby improving ground-truth coverage.
The gain gradually saturates around $B=5$, while further increasing $B$ mainly enlarges the candidate set.
Overall, PCI maintains high ground-truth coverage with compact candidate sets across a broad range of reasonable parameter settings, demonstrating favorable robustness without requiring precise hyperparameter tuning.

\begin{figure}[t]
    \centering
    \includegraphics[width=0.98\linewidth]{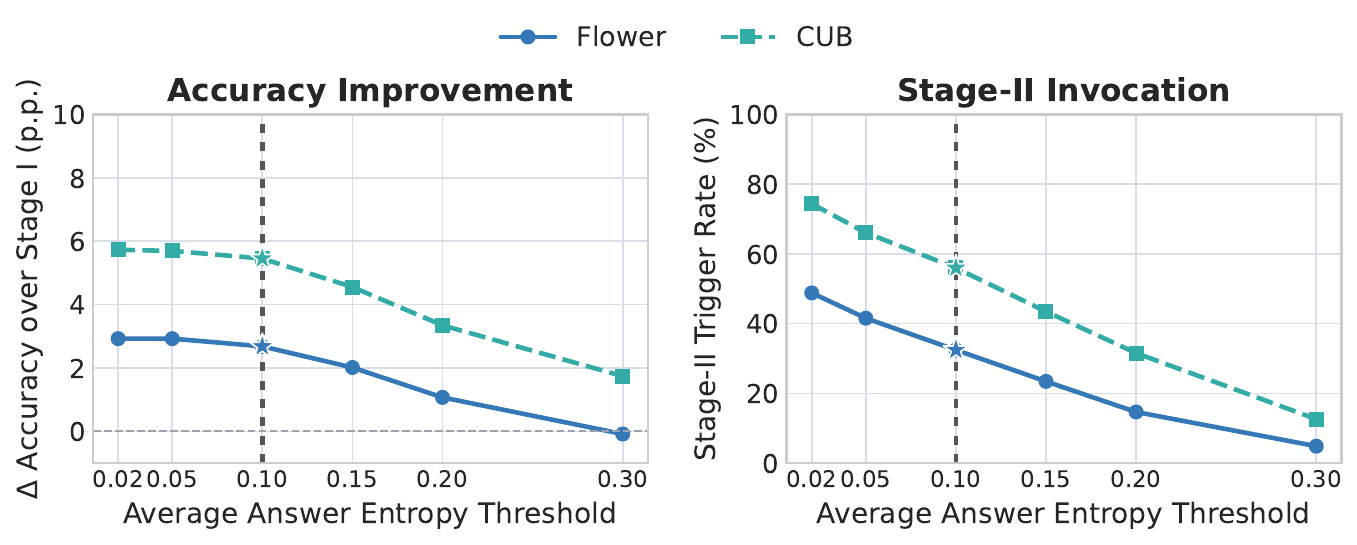}
    \caption{
    Parameter analysis of the entropy threshold $\eta$ in SCR on Flower and CUB.
    The left and right panels report the Top-1 accuracy gain over Stage-I recognition and the Stage-II trigger rate, respectively.
    The dotted lines indicate the default setting $\eta=0.1$.
    }
    \label{fig:scr_parameter}
\end{figure}

\textbf{Parameter Analysis of SCR.}
The entropy threshold $\eta$ in Eq.~\ref{eq:stage2_recognition} determines which queries receive detailed skill guidance according to their recognition uncertainty.
As shown in Fig.~\ref{fig:scr_parameter}, a smaller $\eta$ substantially increases the proportion of queries receiving detailed guidance but brings only limited additional accuracy gains, while an overly large $\eta$ may miss uncertain queries that could benefit from further guidance.
The default $\eta=0.1$ therefore provides a favorable trade-off between recognition improvement and inference efficiency.
More importantly, SCR preserves comparable accuracy gains over a reasonable range of $\eta$, indicating that its effectiveness does not rely on a narrowly tuned threshold.
\section{Conclusion}

In this work, we revisit the FGVR mechanism underlying MLLMs and reveal FGVR Laziness, an inference-time phenomenon where available fine-grained visual cues and category concepts might not be sufficiently associated during recognition.
Through both theoretical and empirical analyses, we demonstrate that such associations would be jointly affected by candidate competition and recognition guidance, while appropriate intervention on these factors can improve the associations and thereby alleviate FGVR Laziness.
Accordingly, we propose LIVE, a training-free framework that constructs compact yet reliable candidate lists and provides adaptive recognition guidance through reusable FGVR skills and on-demand skill invocation.
Extensive experiments on 13 benchmarks spanning both image and video modalities verify the effectiveness of LIVE.
In the future, we plan to investigate whether MLLMs exhibit similar laziness phenomena in visual question answering, visual grounding, and video temporal grounding, while extending LIVE to broader recognition scenarios, including fine-grained text classification.

\bibliographystyle{IEEEtran}
\bibliography{references}

\clearpage

\twocolumn[{%
    \centering
    {\Huge
    Supplementary Material for From Glance to Focus: Mitigating Recognition
    Laziness in MLLMs for Fine-Grained Recognition\par}
    \vspace{1em}

    {\sublargesize
    Yutong Yang, Zhuqi Liu, Mouxing Yang, and Xi Peng\par}
    \vspace{3em}
}]

\begingroup
\renewcommand{\thefootnote}{}
\footnotetext{%
Y. Yang, Z. Liu, and M. Yang are with the College of Computer Science, Sichuan University, Chengdu 610065, China
(e-mail: \{yutongyang.gm, yunfree36, yangmouxing\}@gmail.com).
X. Peng is with the School of Artificial Intelligence, Sichuan University, Chengdu 610065, China,
and also with Tianfu Jincheng Laboratory, Chengdu 610093, China
(e-mail: pengx.gm@gmail.com).
}
\endgroup

\appendices
\setcounter{theorem}{0} 


This supplementary material provides additional theoretical analyses, empirical evidence, and experimental details of LIVE.
Specifically, Appendix~\ref{sec:sup_theory} presents detailed proofs of the three theorems introduced in the main paper.
Appendix~\ref{sec:sup_empirical} provides additional empirical analyses and visualizations of FGVR Laziness.
Appendix~\ref{sec:sup_detail} provides additional experimental details, including dataset statistics, implementation details, prompt templates,  representative examples of Category Skills, and further analysis of PCI.

\section{Theoretical Analysis}
\label{sec:sup_theory}

In this section, we provide detailed proofs of the three theorems introduced in the main paper.
We first formulate category-to-image attention and analyze the effects of candidate ordering and candidate competition, followed by the degradation risks of candidate interventions.

\subsection{Attention Formulation}
\label{sec:attention_formulation}

As discussed in the main paper, we adopt category-to-image attention as a proxy for the association strength between visual cues and category concepts.
Here, we provide its mathematical formulation under causal attention.

For a given category $c$, we use $\mathcal{I}$ and $\mathcal{T}_c$ to represent the image tokens and the tokens corresponding to $c$ in the candidate list, respectively.
For each category token $t\in\mathcal{T}_c$, its attention weight assigned to an accessible token $j$ is given by
\begin{equation}
\alpha_{tj}
=
\frac{\exp(s_{tj})}
{\sum_{m\in\mathcal{P}_t}\exp(s_{tm})},
\qquad
s_{tj}
=
\frac{\mathbf{q}_t^\top\mathbf{k}_j}{\sqrt{d}},
\label{eq:sup_attention}
\end{equation}
where $\mathcal{P}_t$ denotes the set of tokens accessible to $t$ under causal attention, and $\mathbf{q}_t$ and $\mathbf{k}_j$ denote the corresponding query and key representations.
The category-to-image attention of category $c$ is then defined as
\begin{equation}
R_{\Theta}(c\mid x,\mathbf{C},G)
=
\frac{1}{|\mathcal{T}_c|}
\sum_{t\in\mathcal{T}_c}
\sum_{i\in\mathcal{I}}
\alpha_{ti}.
\label{eq:sup_category_attention}
\end{equation}
For notational simplicity, the attention-head and layer indices are omitted.

In the following proofs, we consider causal attention with image tokens preceding the candidate list.
To isolate the effects of candidate interventions, we compare attention distributions while holding the logits to shared accessible tokens fixed.
The target category retains the same tokenization across the compared candidate lists.
These controlled comparisons allow us to examine how changes in candidate accessibility or competing attention mass affect category-to-image attention.

\subsection{Proof of Theorem~1}
\label{sec:proof_ordering}

\begin{theorem}[Candidate Ordering Effect]
\label{sup:thm:ordering}
Under causal attention with image tokens preceding the candidate list and attention logits to shared accessible tokens held fixed, moving a specified category $c$ forward across other candidates increases its category-to-image attention, namely,
\begin{equation}
R_{\Theta}(c\mid x,\mathbf{C}_{c\cdot},G)
>
R_{\Theta}(c\mid x,\mathbf{C}_{\cdot c},G),
\label{sup:eq:ordering_effect}
\end{equation}
where $\mathbf{C}_{c\cdot}$ and $\mathbf{C}_{\cdot c}$ have the same candidate composition but place $c$ at earlier and later positions, respectively.
The relative order of other candidates is preserved, with $x$ and $G$ fixed.
\end{theorem}

\begin{proof}
Consider two candidate lists with identical composition:
\begin{equation}
\begin{aligned}
\mathbf{C}_{c\cdot}
&=(\ldots,c,b_1,\ldots,b_K,\ldots),\\
\mathbf{C}_{\cdot c}
&=(\ldots,b_1,\ldots,b_K,c,\ldots),
\end{aligned}
\label{eq:sup_ordering_lists}
\end{equation}
where $K\geq 1$ and the relative order of all candidates other than $c$ remains unchanged.

For any category token $t\in\mathcal{T}_c$, the tokens corresponding to $b_1,\ldots,b_K$ are inaccessible under $\mathbf{C}_{c\cdot}$ but become accessible under $\mathbf{C}_{\cdot c}$ due to causal attention.
We denote these newly accessible tokens by $\mathcal{B}$ and the shared accessible non-image tokens by $\mathcal{O}_t$.
Since the image tokens precede the candidate list, the corresponding accessible token sets satisfy
\begin{equation}
\begin{aligned}
\mathcal{P}_t(\mathbf{C}_{c\cdot})
&=\mathcal{I}\cup\mathcal{O}_t,\\
\mathcal{P}_t(\mathbf{C}_{\cdot c})
&=\mathcal{I}\cup\mathcal{O}_t\cup\mathcal{B}.
\end{aligned}
\label{eq:sup_ordering_accessibility}
\end{equation}

Following the controlled comparison in Appendix~\ref{sec:attention_formulation}, the attention logits to shared accessible tokens remain unchanged.
For each $t$, define
\begin{equation}
\begin{aligned}
Z_{\mathcal{I},t}
&=\sum_{i\in\mathcal{I}}\exp(s_{ti}),\\
Z_{\mathcal{O},t}
&=\sum_{j\in\mathcal{O}_t}\exp(s_{tj}),\\
Z_{\mathcal{B},t}
&=\sum_{j\in\mathcal{B}}\exp(s_{tj}),
\end{aligned}
\label{eq:sup_ordering_mass}
\end{equation}
where $Z_{\mathcal{B},t}$ is evaluated under $\mathbf{C}_{\cdot c}$.
Since both $\mathcal{I}$ and $\mathcal{B}$ are non-empty, we have $Z_{\mathcal{I},t}>0$ and $Z_{\mathcal{B},t}>0$.

According to Eq.~\ref{eq:sup_attention}, the attention assigned by token $t$ to the image tokens under the two candidate orders is
\begin{equation}
\begin{aligned}
R_t(\mathbf{C}_{c\cdot})
&=
\frac{Z_{\mathcal{I},t}}
{Z_{\mathcal{I},t}+Z_{\mathcal{O},t}},\\
R_t(\mathbf{C}_{\cdot c})
&=
\frac{Z_{\mathcal{I},t}}
{Z_{\mathcal{I},t}+Z_{\mathcal{O},t}+Z_{\mathcal{B},t}}.
\end{aligned}
\label{eq:sup_ordering_token_attention}
\end{equation}

Therefore,
\begin{equation}
\begin{aligned}
&R_t(\mathbf{C}_{c\cdot})
-R_t(\mathbf{C}_{\cdot c})\\
&=
\frac{Z_{\mathcal{I},t}Z_{\mathcal{B},t}}
{(Z_{\mathcal{I},t}+Z_{\mathcal{O},t})
(Z_{\mathcal{I},t}+Z_{\mathcal{O},t}+Z_{\mathcal{B},t})}
>0.
\end{aligned}
\label{eq:sup_ordering_difference}
\end{equation}

Thus, moving category $c$ forward strictly increases the image attention of each token $t\in\mathcal{T}_c$ under the controlled comparison.
Averaging over the category tokens according to Eq.~\ref{eq:sup_category_attention} yields
\begin{equation}
R_{\Theta}(c\mid x,\mathbf{C}_{c\cdot},G)
>
R_{\Theta}(c\mid x,\mathbf{C}_{\cdot c},G),
\end{equation}
which completes the proof.
\end{proof}

\subsection{Proof of Theorem~2}
\label{sec:proof_composition}

\begin{theorem}[Candidate Competition Effect]
\label{sup:thm:composition}
Under causal attention with image tokens preceding the candidate list and attention logits to shared accessible tokens held fixed, reducing attention competition from other candidates increases the category-to-image attention of a specified category $c$, namely,
\begin{equation}
R_{\Theta}(c\mid x,\mathbf{C}_{c+},G)
>
R_{\Theta}(c\mid x,\mathbf{C}_{c-},G),
\label{sup:eq:composition_effect}
\end{equation}
where $\mathbf{C}_{c+}$ and $\mathbf{C}_{c-}$ impose weaker and stronger attention competition on $c$, respectively.
The position of $c$ is preserved, with $x$ and $G$ fixed.
\end{theorem}

\begin{proof}
Consider two candidate lists that differ by replacing $K\geq 1$ candidates preceding the target category $c$:
\begin{equation}
\begin{aligned}
\mathbf{C}_{c-}
&=(\ldots,b_1,\ldots,b_K,c,\ldots),\\
\mathbf{C}_{c+}
&=(\ldots,b'_1,\ldots,b'_K,c,\ldots),
\end{aligned}
\label{eq:sup_competition_lists}
\end{equation}
where the replaced candidates have matching token lengths, preserving the position of $c$ and the remaining candidate composition and order.

For any category token $t\in\mathcal{T}_c$, the replaced candidates appear before $c$ and are therefore accessible under causal attention.
We denote their token sets under $\mathbf{C}_{c-}$ and $\mathbf{C}_{c+}$ by $\mathcal{B}^{-}$ and $\mathcal{B}^{+}$, respectively, and the shared accessible non-image tokens by $\mathcal{O}_t$.

Following the controlled comparison in Appendix~\ref{sec:attention_formulation}, the attention logits to shared accessible tokens remain unchanged.
For each $t$, define
\begin{equation}
\begin{aligned}
Z_{\mathcal{I},t}
&=\sum_{i\in\mathcal{I}}\exp(s_{ti}),\\
Z_{\mathcal{O},t}
&=\sum_{j\in\mathcal{O}_t}\exp(s_{tj}),\\
Z_t^{\pm}
&=\sum_{j\in\mathcal{B}^{\pm}}\exp(s_{tj}^{\pm}),
\end{aligned}
\label{eq:sup_competition_mass}
\end{equation}
where $s_{tj}^{\pm}$ denotes the attention logit under the corresponding candidate list.
Here, $Z_t^{\pm}$ measures the attention competition contributed by the replaced candidates.

Accordingly, weaker attention competition under $\mathbf{C}_{c+}$ means
\begin{equation}
Z_t^{+}\leq Z_t^{-},
\qquad \forall t\in\mathcal{T}_c,
\label{eq:sup_competition_condition}
\end{equation}
with strict inequality for at least one category token.

According to Eq.~\ref{eq:sup_attention}, the image attention of token $t$ under the two candidate compositions is
\begin{equation}
R_t(\mathbf{C}_{c\pm})
=
\frac{Z_{\mathcal{I},t}}
{Z_{\mathcal{I},t}+Z_{\mathcal{O},t}+Z_t^{\pm}}.
\label{eq:sup_competition_token_attention}
\end{equation}

To examine the effect of candidate competition, we express the image attention as a function of the competing attention mass $Z$:
\begin{equation}
R_t(Z)
=
\frac{Z_{\mathcal{I},t}}
{Z_{\mathcal{I},t}+Z_{\mathcal{O},t}+Z}.
\label{eq:sup_competition_function}
\end{equation}
Taking its derivative with respect to $Z$ gives
\begin{equation}
\frac{\partial R_t}{\partial Z}
=
-\frac{Z_{\mathcal{I},t}}
{(Z_{\mathcal{I},t}+Z_{\mathcal{O},t}+Z)^2}
<0,
\label{eq:sup_competition_derivative}
\end{equation}
since $Z_{\mathcal{I},t}>0$.

Therefore, reducing the competing attention mass increases the attention allocated to the image tokens.
Under Eq.~\ref{eq:sup_competition_condition}, we have
\begin{equation}
R_t(\mathbf{C}_{c+})
\geq
R_t(\mathbf{C}_{c-}),
\qquad \forall t\in\mathcal{T}_c,
\label{eq:sup_competition_token_result}
\end{equation}
with strict inequality for at least one category token.
Averaging over all category tokens according to Eq.~\ref{eq:sup_category_attention} yields
\begin{equation}
R_{\Theta}(c\mid x,\mathbf{C}_{c+},G)
>
R_{\Theta}(c\mid x,\mathbf{C}_{c-},G),
\end{equation}
which completes the proof.
\end{proof}
The above analysis also explains why candidate removal can yield a greater attention gain than candidate replacement.
Under the same controlled setting, removing the competing candidates eliminates their contributions to the attention denominator, whereas replacing them introduces a positive competing attention mass.
For a category token $t\in\mathcal{T}_c$, we therefore have
\begin{equation}
\begin{aligned}
R_t^{\mathrm{rem}}
&=
\frac{Z_{\mathcal{I},t}}
{Z_{\mathcal{I},t}+Z_{\mathcal{O},t}}\\
&>
\frac{Z_{\mathcal{I},t}}
{Z_{\mathcal{I},t}+Z_{\mathcal{O},t}+Z_t^{\mathrm{rep}}}
=
R_t^{\mathrm{rep}},
\end{aligned}
\label{eq:sup_removal_effect}
\end{equation}
where $Z_t^{\mathrm{rep}}>0$ denotes the attention mass contributed by the replacement candidates.
Averaging over all category tokens yields the same inequality for $R_{\Theta}$.
Thus, removing competing candidates maximizes the category-to-image attention among the corresponding removal and replacement interventions, provided that the target category is retained.

\subsection{Proof of Theorem~3}
\label{sec:proof_intervention}

\begin{theorem}[Degradation Risk of Candidate Interventions]
\label{sup:thm:risk}
Consider a uniform distribution over the $M$ candidate positions.
Move-to-front reordering has an average degradation risk of $\frac{M-1}{2M}$ over all possible starting positions, whereas order-preserving candidate reduction has a risk of $1-p_{\mathrm{ret}}$.
Therefore, candidate reduction has lower degradation risk if and only if
\begin{equation}
p_{\mathrm{ret}}>\frac{1}{2}+\frac{1}{2M},
\label{sup:eq:risk_comparison}
\end{equation}
where $p_{\mathrm{ret}}$ denotes the ground-truth retention probability.
\end{theorem}

\begin{proof}
Consider a candidate list $\mathbf{C}=(c_1,\ldots,c_M)$ containing the ground-truth category $y$.
We define the degradation risk of candidate reordering as the probability that $y$ is shifted backward, and that of candidate reduction as the probability that $y$ is excluded.
Since the ground-truth position is unknown, we consider a uniform distribution over the $M$ candidate positions:
\begin{equation}
P(y=c_i)=\frac{1}{M},
\qquad i=1,\ldots,M.
\label{eq:sup_uniform_position}
\end{equation}

We first consider move-to-front reordering.
Suppose the candidate at position $k$ is moved to the first position while preserving the relative order of the remaining candidates.
This operation shifts exactly the preceding $k-1$ candidates backward by one position.
Therefore, the degradation risk for this operation is
\begin{equation}
D_{\mathrm{ord}}(k)
=
P(y\in\{c_1,\ldots,c_{k-1}\})
=
\frac{k-1}{M}.
\label{eq:sup_reordering_risk}
\end{equation}
Averaging over all possible starting positions $k=1,\ldots,M$ gives
\begin{equation}
\begin{aligned}
D_{\mathrm{ord}}
&=
\frac{1}{M}\sum_{k=1}^{M}D_{\mathrm{ord}}(k)\\
&=
\frac{1}{M^2}\sum_{k=1}^{M}(k-1)\\
&=
\frac{M-1}{2M}.
\end{aligned}
\label{eq:sup_average_reordering_risk}
\end{equation}

We next consider order-preserving candidate reduction.
Given a reduced candidate list $\mathbf{C}'\subseteq\mathbf{C}$, the relative order of the retained candidates remains unchanged.
Thus, the ground-truth category cannot be shifted backward if retained, and the degradation event occurs only when $y$ is excluded.
With the ground-truth retention probability defined as $p_{\mathrm{ret}}=P(y\in\mathbf{C}')$, the degradation risk is
\begin{equation}
D_{\mathrm{red}}
=
P(y\notin\mathbf{C}')
=
1-p_{\mathrm{ret}}.
\label{eq:sup_reduction_risk}
\end{equation}

Finally, candidate reduction has lower degradation risk than move-to-front reordering if and only if
\begin{equation}
\begin{aligned}
D_{\mathrm{red}}<D_{\mathrm{ord}}
&\iff
1-p_{\mathrm{ret}}<\frac{M-1}{2M}\\
&\iff
p_{\mathrm{ret}}>\frac{1}{2}+\frac{1}{2M}.
\end{aligned}
\end{equation}
This completes the proof.
\end{proof}

\section{Empirical Analysis}
\label{sec:sup_empirical}

In this section, we provide additional empirical analyses of FGVR Laziness.
Specifically, we first examine the effects of candidate ordering on category-to-image attention and recognition accuracy, and then visualize how candidate reduction alters attention allocation and recovers incorrect predictions.

\begin{figure*}[t]
    \centering
    \includegraphics[width=0.98\linewidth]{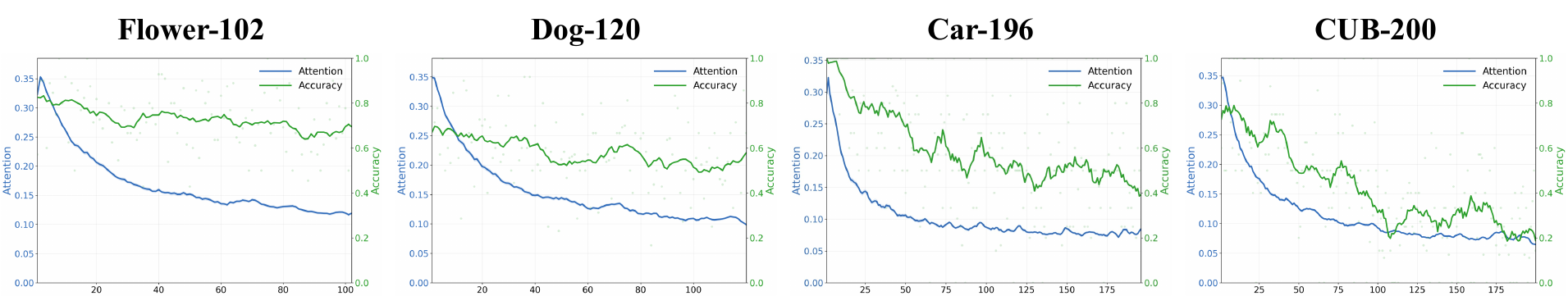}
    \caption{
    Empirical analysis of the Candidate Ordering Effect.
    For each dataset, we randomly sample 1,500 test images and independently shuffle the complete candidate list for each image.
    The blue curve shows the average category-to-image attention of the ground-truth category at each position, while the green curve shows the corresponding recognition accuracy.
    }
    \label{fig:position_effect}
\end{figure*}

\subsection{Empirical Analysis of Candidate Ordering Effect}

To empirically examine the Candidate Ordering Effect, we conduct a controlled position randomization experiment on four representative FGVR datasets using Qwen2.5-VL-7B.
For each dataset, we randomly sample 1,500 test images and independently shuffle the complete candidate list for each image while preserving its composition.
We measure category-to-image attention using ground-truth category tokens as queries and image patch tokens as keys.
Specifically, we first average the attention weights across category tokens and attention heads within each layer.
To capture visual regions receiving high category-related attention at different layers, we retain the maximum attention for each image patch across all 28 layers and sum the resulting values over image patches.
We then examine the category-to-image attention and recognition accuracy at different ground-truth category positions.

As shown in Fig.~\ref{fig:position_effect}, category-to-image attention generally decreases as the ground-truth category moves backward, with a more pronounced decline at earlier positions.
Recognition accuracy exhibits a similar downward trend, particularly on Cars-196 and CUB-200.
Since the query image, candidate composition, and recognition guidance remain unchanged, these results empirically support the Candidate Ordering Effect analyzed in Appendix~\ref{sec:sup_theory}, showing that candidate ordering alone can affect both attention allocation and recognition outcomes.
More importantly, considering category-to-image attention as a proxy for visual-conceptual association, these observations suggest that MLLMs might fail to consistently associate visual cues with the ground-truth category concept under different candidate arrangements, providing further evidence of FGVR Laziness.

\subsection{Visualization of FGVR Laziness}

\begin{figure*}[t]
    \centering
    \includegraphics[width=0.98\linewidth]{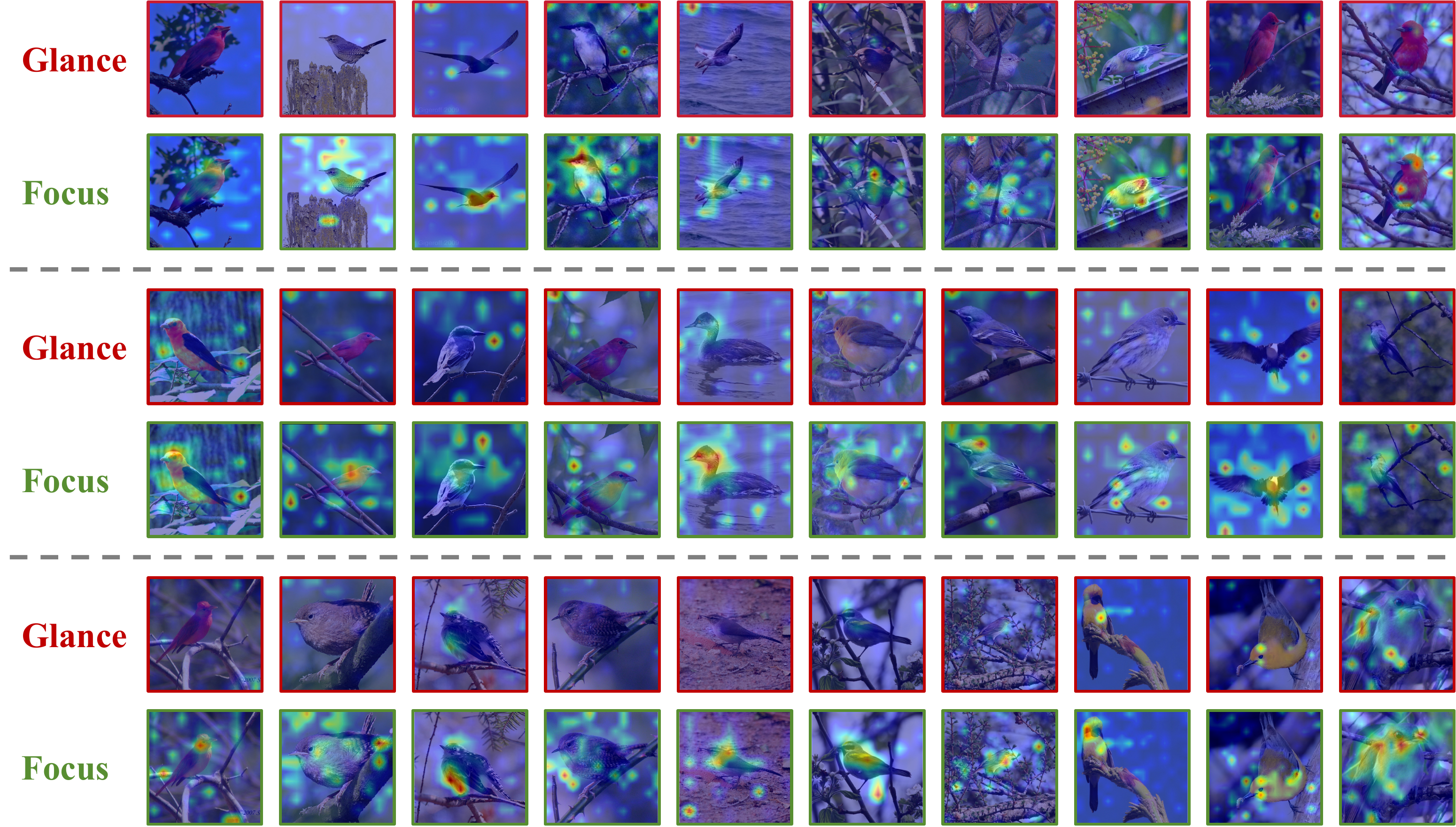}
    \caption{
    Visualization of FGVR Laziness on CUB-200.
    We show representative samples misclassified by Qwen2.5-VL-7B under full-category recognition but correctly recognized after retaining only the ground-truth category and its original incorrect prediction.
    Glance (red) and Focus (green) correspond to recognition under the complete and reduced candidate lists, respectively, with attention maps showing category-to-image attention for the ground-truth category.
    Compared with Glance, Focus generally exhibits stronger and more concentrated attention on the target bird.
    }
    \label{fig:laziness_visualization}
\end{figure*}

To further visualize FGVR Laziness, we examine recoverable recognition errors on CUB-200 using Qwen2.5-VL-7B.
Specifically, we collect samples misclassified under full-category recognition and re-evaluate them after retaining only the ground-truth category and the original incorrect prediction.
Using the same attention aggregation as in the preceding analysis, Fig.~\ref{fig:laziness_visualization} presents representative examples of error recovery, comparing the category-to-image attention of the ground-truth category under the complete (Glance) and reduced (Focus) candidate lists.

Taking the first sample as an example, the model initially misclassifies Summer Tanager as Vermilion Flycatcher, with relatively weak attention from the ground-truth category to the target bird.
After reducing the candidate list to these two categories, the prediction is corrected to Summer Tanager, accompanied by stronger and more concentrated attention on the bird.
Similar patterns are observed in the remaining examples.
Notably, since both the ground-truth category and the original incorrect prediction are retained, these corrections suggest that the initial errors cannot be attributed solely to an inability to distinguish the two categories.
Instead, the accompanying attention changes suggest that the original candidate context might hinder the association between discriminative visual cues and the ground-truth category concept, providing further empirical evidence of FGVR Laziness.
These findings suggest that strengthening the association between visual cues and category concepts to mitigate FGVR Laziness could be a promising approach to improving the FGVR performance of MLLMs at inference time.

\section{Experimental Details and Additional Analysis}
\label{sec:sup_detail}

\subsection{Datasets}
\label{sec:supp_datasets}

We conduct experiments on 13 benchmarks covering three recognition scenarios.
For fine-grained image recognition, we use Oxford Flowers-102~\cite{flower}, CUB-200-2011~\cite{cub}, Oxford-IIIT Pet~\cite{pet}, FGVC-Aircraft~\cite{aircraft}, Stanford Cars~\cite{car}, Stanford Dogs~\cite{dog}, DTD~\cite{dtd}, Food-101~\cite{food}, and SUN397~\cite{sun}.
For fine-grained action recognition, we use UCF101~\cite{ucf101} and HMDB51~\cite{hmdb51}.
For general and out-of-distribution (OOD) image recognition, we use ImageNet-1K~\cite{imagenet-1k} and ImageNet-Sketch~\cite{imagenet-sketch}, respectively.
The dataset statistics are summarized in Table~\ref{tab:dataset_statistics}.

We follow the official dataset splits, using split 1 for datasets with multiple predefined splits.
Unless otherwise specified, we evaluate on the corresponding test sets and randomly sample four reference samples per category from the training sets using a fixed seed of 42.
The same 4-shot reference set is reused throughout all reference-based procedures, including PCI and CSC.
For ImageNet-Sketch, which lacks an official training split, we sample reference images from the complete dataset and evaluate on the remaining images, as detailed in Table~\ref{tab:dataset_statistics}.

\begin{table}[t]
\centering
\caption{Statistics of the 13 benchmarks used in our experiments.}
\label{tab:dataset_statistics}

\renewcommand{\arraystretch}{1.15}
\setlength{\tabcolsep}{8pt}

\begin{tabular*}{0.95\columnwidth}{@{\extracolsep{\fill}}l r r l}
\hline
Dataset & Classes & Test Samples & Modality \\
\hline
Oxford Flowers-102~\cite{flower}       & 102   & 6,149  & Image \\
CUB-200-2011~\cite{cub}                & 200   & 5,794  & Image \\
Oxford-IIIT Pet~\cite{pet}             & 37    & 3,669  & Image \\
FGVC-Aircraft~\cite{aircraft}          & 100   & 3,333  & Image \\
Stanford Cars~\cite{car}               & 196   & 8,041  & Image \\
Stanford Dogs~\cite{dog}               & 120   & 8,580  & Image \\
DTD~\cite{dtd}                         & 47    & 1,880  & Image \\
Food-101~\cite{food}                   & 101   & 25,250 & Image \\
SUN397~\cite{sun}                      & 397   & 19,850 & Image \\
\hline
UCF101~\cite{ucf101}                   & 101   & 3,783  & Video \\
HMDB51~\cite{hmdb51}                   & 51    & 1,530  & Video \\
\hline
ImageNet-1K~\cite{imagenet-1k}         & 1,000 & 50,000 & Image \\
ImageNet-Sketch$^\dagger$~\cite{imagenet-sketch} & 1,000 & 46,889 & Image \\
\hline
\end{tabular*}

\vspace{1mm}
\parbox{0.95\columnwidth}{\footnotesize
$^\dagger$ImageNet-Sketch does not provide an official training split.
We therefore sample four reference images per category from the complete dataset and evaluate on the remaining images.}
\end{table}

\begin{figure}[!t]
    \centering
    \begin{promptbox}
        [Category-Specific Visual Knowledge Construction]
        {PromptGrayBg}
        {PromptGrayFrame}

You are building a reusable category skill for fine-grained image recognition.

\vspace{0.5mm}
This skill will be used to distinguish the target category from visually similar categories.
So the output should be concise, precise, and recognition-oriented, not a general description.

\vspace{0.5mm}
Reference images of the same target category are provided.
Carefully examine the images and extract stable category-level visual knowledge that will later help distinguish this category from visually similar categories.
Do not caption the images individually.

\vspace{0.5mm}
\textbf{Dataset:} \texttt{[DATASET]}

\textbf{Target category name:} \texttt{[TARGET CATEGORY]}

\vspace{0.5mm}
\textbf{Write two fields:}

\begin{itemize}
    \setlength{\itemsep}{0pt}
    \setlength{\parskip}{0pt}
    \setlength{\topsep}{1pt}
    \item \textbf{description}: one concise English sentence giving the stable category prototype.
    It should capture the overall appearance and the most stable visual anchor.
    \item \textbf{visual\_checklist}: 1 to 3 English items, ordered from most important to least important.
    Each item must be an instruction-style recognition checkpoint.
    First state what region, part, or structure to inspect, then state the most distinctive evidence of this category there.
\end{itemize}

\textbf{Preferred checklist style:}

\texttt{"Check [region / structure]:}

\texttt{[TARGET CATEGORY] [distinctive evidence]."}

\vspace{0.5mm}
Use only stable cross-image evidence.
Prefer stable and discriminative visual cues.
Avoid vague or generic statements.
Do not include background, habitat, function, or other non-visual knowledge.
Avoid vague statements, background, lighting, viewpoint, crop, and other image-specific factors.

\vspace{0.5mm}
\textbf{Return JSON with this schema:}

{\ttfamily
\{\\
\hspace*{2mm}"description": "...",\\
\hspace*{2mm}"visual\_checklist": ["...", "..."]\\
\}
}

    \end{promptbox}
    \caption{Prompt for constructing category-specific visual knowledge in CSC.}
    \label{fig:app_csc_core}
\end{figure}

\begin{figure}[!t]
    \centering
    \begin{promptbox}
        [Comparative Knowledge Construction]
        {PromptGrayBg}
        {PromptGrayFrame}

You are building comparative skill knowledge for fine-grained image recognition.

\vspace{0.5mm}
\texttt{image\_1} is the true category:\texttt{[TRUE CATEGORY]}

\texttt{image\_2} is the confused category:\texttt{[CONFUSED CATEGORY]}

\vspace{0.5mm}
The model often confuses \texttt{[TRUE CATEGORY]} with \texttt{[CONFUSED CATEGORY]}.
Write one grounded correction rule for \texttt{[TRUE CATEGORY]}.

\vspace{0.5mm}
Use only visible evidence from these two images.
Keep only the strongest 1 or 2 discrimination cues.
Each cue must be tied to a specific visible region, part, or structure.
Do not repeat the general description of the true category.
Do not use vague comparisons without concrete visual evidence.

\vspace{0.5mm}
Write exactly one sentence in this form:

\texttt{"Compared with [CONFUSED CATEGORY],}

\texttt{[TRUE CATEGORY] shows ..., while}

\texttt{[CONFUSED CATEGORY] shows ..."}

\vspace{0.5mm}
\textbf{Return JSON only:}

{\ttfamily
\{\\
\hspace*{2mm}"compare\_true\_to\_confused": "..."\\
\}
}

    \end{promptbox}
    \caption{Prompt for constructing confusion-specific comparative knowledge in CSC.}
    \label{fig:app_csc_cmp}
\end{figure}
\subsection{Implementation Details}
\label{sec:supp_implementation}

Unless otherwise specified, all experiments are conducted under the shared 4-shot reference setting described in Appendix~\ref{sec:supp_datasets}.
For PCI, we adopt CLIP ViT-B/16~\cite{clip} and SigLIP-base-patch16-224~\cite{siglip} as the VLM backbones, with the corresponding variants denoted as $\mathrm{LIVE}_c$ and $\mathrm{LIVE}_s$, respectively.
We use cosine similarity for CLIP and native image--text similarity logits for SigLIP, setting the ranking-reliability threshold $\gamma$ to $0.035$ and $2.5$, respectively, to accommodate their different score scales.
We set $B=5$ for affinity construction and multi-seed candidate inference, the support threshold to $\rho=0.6$, and the candidate budget to $K=20$.

For CSC, each MLLM independently constructs its own Category Skills using the same 4-shot reference set as PCI.
Specifically, one reference sample is randomly selected from each category to construct category-specific visual knowledge, while confusion-specific comparative knowledge is derived from one reference sample of each category in the corresponding confusion pair.
For SCR, prediction uncertainty is measured using average answer entropy, with the threshold $\eta$ set to $0.1$ for Qwen2.5-VL-7B and $0.01$ for Qwen3-VL-8B.
For uncertain predictions, confused categories are identified through prefix matching based on the first-token distribution.

All model parameters remain frozen throughout LIVE.
For video datasets, we use the middle frame for VLM-based candidate inference, while the original video is directly provided to the MLLM using its default video-processing configuration without additional task-specific preprocessing.

\subsection{Prompt Templates}
\label{app:prompts}

For reproducibility, we provide the prompt templates used in Category Skill Construction (CSC) and Skill-inspired Confusion Resolution (SCR).
CSC uses two prompts to construct category-specific visual knowledge and confusion-specific comparative knowledge, while SCR performs two-stage recognition with different levels of Category Skill guidance.
For clarity, we present the image-based prompts used for FGVR.
The prompts for fine-grained action recognition follow the same construction logic, with image-specific visual cues replaced by action-related evidence such as motion patterns, temporal dynamics, and object interactions.
Runtime inputs are represented by bracketed placeholders below.

\subsubsection{Prompts for CSC}
\label{app:prompt_csc}

CSC constructs each Category Skill from complementary category-specific visual knowledge and confusion-specific comparative knowledge.
For each category, one reference image is randomly selected from the shared 4-shot reference set to construct its description and visual checklist using the prompt in Fig.~\ref{fig:app_csc_core}.
For each confusion pair identified from this reference set, one reference image is randomly selected from each category to construct comparative knowledge using the prompt in Fig.~\ref{fig:app_csc_cmp}.
All reference images used in CSC are drawn from the same fixed 4-shot reference set used by PCI, without introducing additional reference data.

\begin{figure}[!t]
    \centering
    \begin{promptbox}
        [Stage-I Recognition]
        {PromptGrayBg}
        {PromptGrayFrame}

You are an expert in fine-grained visual recognition.

\vspace{0.5mm}
You are good at distinguishing visually similar categories by carefully comparing subtle visual cues such as shape, structure, parts, texture, markings, and color patterns.

\vspace{0.5mm}
\textbf{Task:}

I will provide an image and a set of candidate labels.
Your task is to choose the single label that best matches the image.

\vspace{0.5mm}
Carefully compare the candidate labels using fine-grained visual cues.

\vspace{0.5mm}
\textbf{You must follow these rules strictly:}

\begin{enumerate}
    \setlength{\itemsep}{0pt}
    \setlength{\parskip}{0pt}
    \setlength{\topsep}{1pt}
    \item Select exactly one label from the candidate labels.
    \item Do not output anything except the selected label.
    \item Copy the label exactly as it appears in the candidate labels.
\end{enumerate}

Now, carefully observe this image and choose from the following candidate labels:

\vspace{0.5mm}
\texttt{[CANDIDATE LABELS]}

\vspace{0.5mm}
\texttt{[CATEGORY SKILL HEADERS]}

    \end{promptbox}
    \caption{Prompt for Stage-I recognition in SCR, where \texttt{[CATEGORY SKILL HEADERS]} contains the category name and description of each candidate.}
    \label{fig:app_scr_stage1}
\end{figure}

\begin{figure}[!t]
    \centering
    \begin{promptbox}
        [Stage-II Recognition]
        {PromptGrayBg}
        {PromptGrayFrame}

You are an expert in fine-grained visual recognition.

\vspace{0.5mm}
Carefully observe the given image and then choose the most appropriate category from the candidate list.

\vspace{0.5mm}
Output exactly one label copied from the candidate list.

\vspace{0.5mm}
Choose from the candidate classes:

\texttt{[CANDIDATE LABELS]}

\vspace{0.5mm}
\texttt{[CATEGORY SKILLS]}

    \end{promptbox}
    \caption{
    Prompt for Stage-II recognition in SCR.
    \texttt{[CATEGORY SKILLS]} contains the description of each candidate, while the visual checklist and available comparative knowledge are additionally provided for the identified confused categories.
    }
    \label{fig:app_scr_stage2}
\end{figure}

\subsubsection{Prompts for SCR}
\label{app:prompt_scr}

SCR performs recognition in two stages.
Stage-I uses only the skill headers, consisting of category names and descriptions, for lightweight recognition, as shown in Fig.~\ref{fig:app_scr_stage1}.
The average answer entropy is then used to estimate prediction uncertainty.
For uncertain samples, confused categories are identified from the first-token distribution through prefix matching.
Stage-II retains the same candidate list and descriptions of all candidates, while selectively appending visual checklists and available comparative knowledge only for the identified confused categories, as shown in Fig.~\ref{fig:app_scr_stage2}.
In particular, comparative knowledge is included only when the corresponding compared category is also among the identified confused categories.

\begin{figure}[!t]
    \centering
    \begin{promptbox}
        [Example Category Skill]
        {PromptGrayBg}
        {PromptGrayFrame}

\begin{minipage}[t]{0.27\linewidth}
\vspace{0pt}
    \centering
    \includegraphics[width=\linewidth]{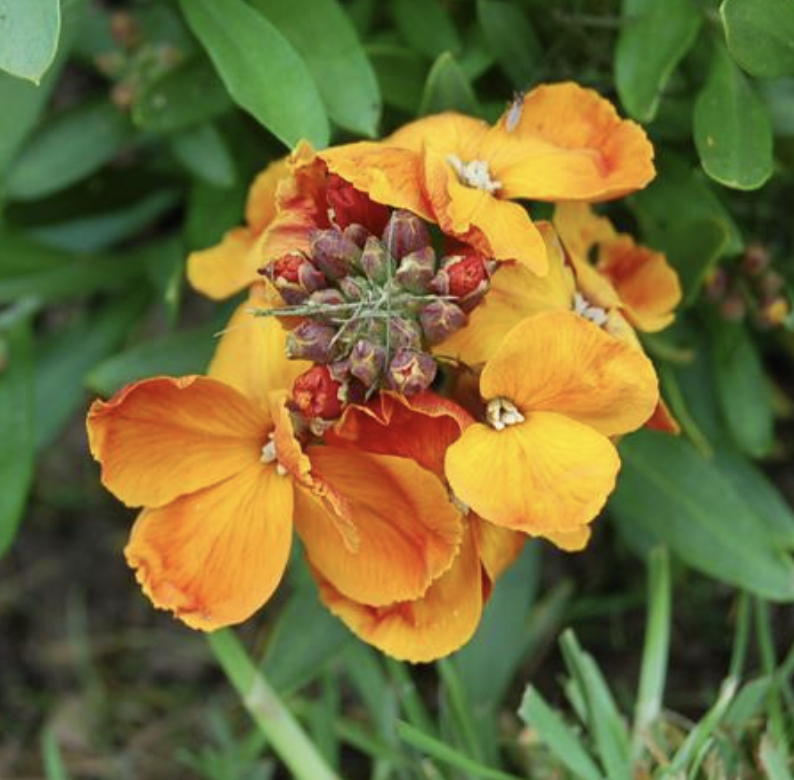}
\end{minipage}
\hfill
\begin{minipage}[t]{0.68\linewidth}
\vspace{0pt}

\textbf{\textcolor{DeepOrange}{Skill Header}}

\noindent\textcolor{black!18}{\rule{\linewidth}{0.35pt}}

\vspace{0.8mm}
\textbf{\textcolor{DeepOrange}{Category}:} wallflower

\vspace{0.8mm}
\textbf{\textcolor{DeepOrange}{Description}:}

wallflower with vibrant orange petals and a central cluster of small, reddish-brown buds.

\end{minipage}

\vspace{1.5mm}

\textbf{\textcolor{DeepOrange}{Skill Body}}

\noindent\textcolor{black!18}{\rule{\linewidth}{0.35pt}}

\vspace{0.8mm}
\textbf{\textcolor{DeepOrange}{Visual checklist}:}

\begin{itemize}
    \setlength{\itemsep}{1pt}
    \setlength{\parskip}{0pt}
    \setlength{\topsep}{1pt}
    \setlength{\leftmargin}{5mm}
    \item Check the flower: vibrant orange petals with a central bud cluster.
    \item Check the bud structure: small, reddish-brown buds tightly grouped in a central cluster.
    \item Check the petal arrangement: four large, overlapping petals forming a cup-like shape.
\end{itemize}

\vspace{0.5mm}

\textbf{\textcolor{DeepOrange}{Comparative knowledge}:}

Compared with \textit{californian poppy}, \textit{wallflower} shows a more densely clustered and layered petal arrangement, while \textit{californian poppy} displays a single, broad, and open petal structure.

    \end{promptbox}
    \caption{
    A representative Category Skill constructed by CSC using Qwen2.5-VL-7B for \textit{wallflower}.
    The skill combines category-specific visual knowledge with comparative knowledge for distinguishing \textit{wallflower} from its confused category \textit{californian poppy}.
    }
    \label{fig:app_category_skill_qwen25}
\end{figure}

\begin{figure}[!t]
    \centering
    \begin{promptbox}
        [Example Category Skill]
        {PromptGrayBg}
        {PromptGrayFrame}

\begin{minipage}[t]{0.27\linewidth}
\vspace{0pt}
    \centering
    \includegraphics[width=\linewidth]{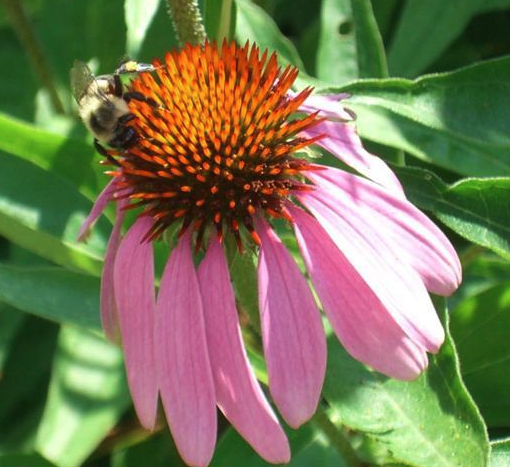}
\end{minipage}
\hfill
\begin{minipage}[t]{0.68\linewidth}
\vspace{0pt}

\textbf{\textcolor{DeepPink}{Skill Header}}

\noindent\textcolor{black!18}{\rule{\linewidth}{0.35pt}}

\vspace{0.8mm}
\textbf{\textcolor{DeepPink}{Category}:} purple coneflower

\vspace{0.8mm}
\textbf{\textcolor{DeepPink}{Description}:}

A flower with pinkish-purple ray florets surrounding a prominent, spiky orange-brown central cone.

\end{minipage}

\vspace{1.5mm}

\textbf{\textcolor{DeepPink}{Skill Body}}

\noindent\textcolor{black!18}{\rule{\linewidth}{0.35pt}}

\vspace{0.8mm}
\textbf{\textcolor{DeepPink}{Visual checklist}:}

\begin{itemize}
    \setlength{\itemsep}{1pt}
    \setlength{\parskip}{0pt}
    \setlength{\topsep}{1pt}
    \setlength{\leftmargin}{5mm}
    \item Check center: purple coneflower has a dense, spiky orange-brown cone.
    \item Check petals: purple coneflower has drooping, pinkish-purple ray florets.
\end{itemize}

\vspace{0.5mm}

\textbf{\textcolor{DeepPink}{Comparative knowledge}:}

Compared with \textit{blanket flower}, \textit{purple coneflower} shows a prominent, spiky orange cone center, while \textit{blanket flower} shows a flatter, yellow-centered disk surrounded by red-and-yellow bicolored petals.

    \end{promptbox}
    \caption{
    A representative Category Skill constructed by CSC using Qwen3-VL-8B for \textit{purple coneflower}.
    The skill combines category-specific visual knowledge with comparative knowledge for distinguishing \textit{purple coneflower} from its confused category \textit{blanket flower}.
    }
    \label{fig:app_category_skill_qwen3}
\end{figure}

\subsection{Examples of Category Skills}
\label{sec:supp_skill_examples}

To provide a concrete view of the Category Skills constructed by CSC, we present two representative examples generated by Qwen2.5-VL-7B and Qwen3-VL-8B in Figs.~\ref{fig:app_category_skill_qwen25} and~\ref{fig:app_category_skill_qwen3}, respectively.
Each Category Skill contains a lightweight skill header consisting of the category name and description, together with a detailed skill body comprising a visual checklist and confusion-specific comparative knowledge.

\subsection{Additional Analysis of PCI}
\label{sec:supp_pci_analysis}

\begin{figure}[t]
    \centering
    \includegraphics[width=0.47\linewidth]{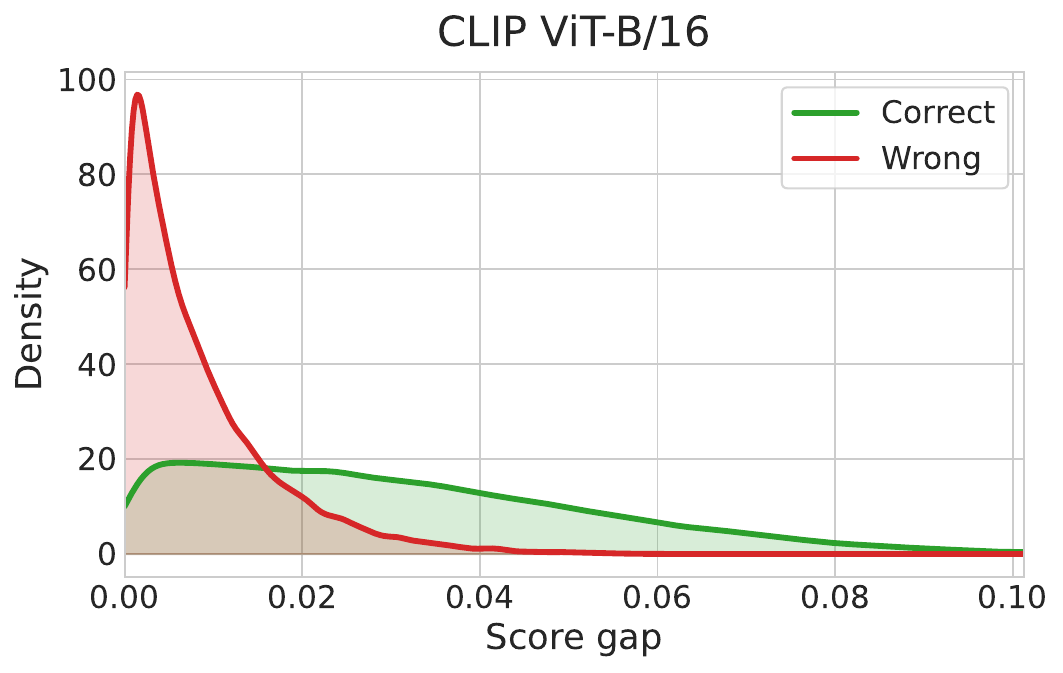}
    \hfill
    \includegraphics[width=0.47\linewidth]{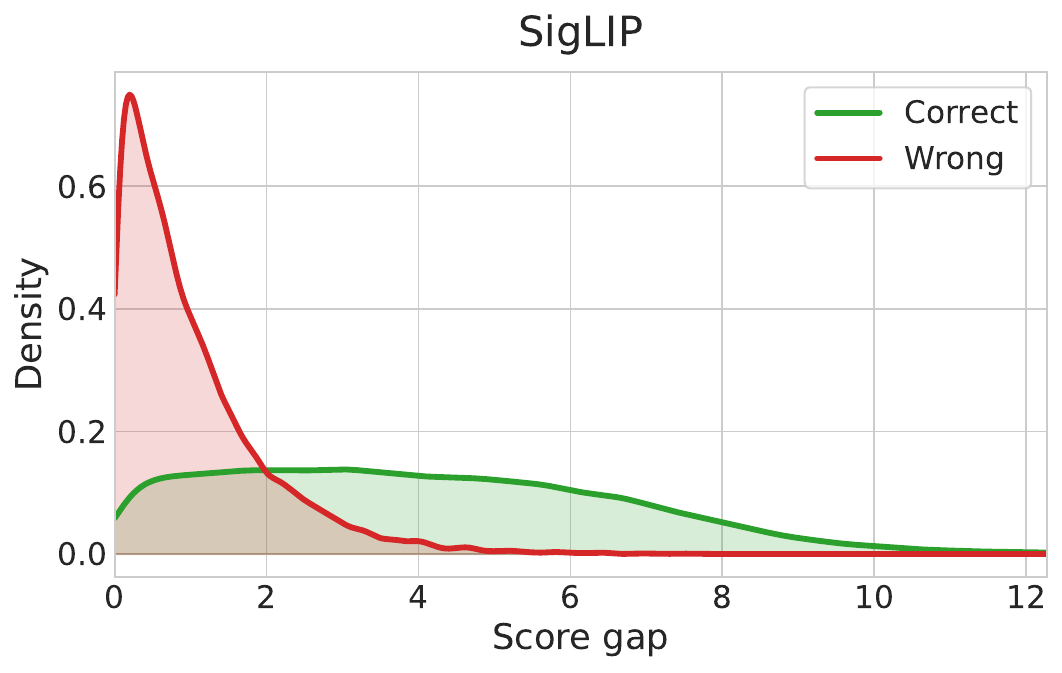}
    \caption{
    Score-gap distributions for correct and incorrect VLM predictions over the nine FGVR datasets.
    The score gap is computed between the top-1 and top-2 predictions for CLIP ViT-B/16 (left) and SigLIP (right).
    }
    \label{fig:pci_gap}
\end{figure}

To further examine the ranking-reliability criterion in PCI, we compare the score-gap distributions of correct and incorrect VLM predictions over the nine FGVR datasets.
As shown in Fig.~\ref{fig:pci_gap}, incorrect predictions from both CLIP ViT-B/16 and SigLIP are concentrated in the small-gap region, whereas correct predictions extend more broadly toward larger gaps.
This result supports the score gap as an effective signal for adapting candidate inference according to VLM ranking reliability.
We set $\gamma$ to $0.035$ for CLIP ViT-B/16 and $2.5$ for SigLIP, and keep them fixed across all datasets.

\end{document}